\documentclass{article}
\usepackage[sort]{cite}
\usepackage[utf8]{inputenc} 
\usepackage{geometry} 
\usepackage{graphicx} 

\usepackage{mathtools}
\usepackage{mathrsfs}
\usepackage{amsmath, amssymb, amsopn, amsthm} 
\usepackage{lipsum}
\usepackage{amsfonts}
\usepackage{graphicx}
\usepackage{epstopdf}
\usepackage{algorithm, algorithmicx, algpseudocode}
\usepackage{booktabs} 
\usepackage{array} 
\usepackage{paralist} 
\usepackage{verbatim} 
\usepackage{subcaption} 
\usepackage{float}
\usepackage[usenames,dvipsnames]{xcolor} 
\usepackage{mathtools}
\usepackage{mathrsfs}
\usepackage{multirow}
\usepackage{mdwlist}
\usepackage[colorlinks]{hyperref}
\definecolor{GGreen}{RGB}{0,128,0}
\hypersetup{
    linkcolor= {GGreen},
    citecolor={blue}, 
    urlcolor={cyan}
}
\usepackage[nameinlink]{cleveref}
\crefformat{equation}{(#2#1#3)}

\newcommand{\mH}{$(\mathscr{H})$}
\newcommand{\mHo}{$(\mathscr{H}_1)$} 

\title{\bfseries Joint reconstruction-segmentation on graphs}
\author{Jeremy Budd\thanks{Hausdorff Centre for Mathematics, Rheinische Friedrich-Wilhelms-Universit\"at Bonn, Germany (\texttt{jeremy.budd@hcm.uni-bonn.de})} \and 
	Yves van Gennip\thanks{Delft Institute of Applied Mathematics, Technische Universiteit Delft, The Netherlands (\texttt{y.vangennip@tudelft.nl})}  \and
	 Jonas Latz\thanks{Maxwell Institute for Mathematical Sciences \& School of Mathematical and Computer Sciences, Heriot-Watt University, Edinburgh, United Kingdom (\texttt{j.latz@hw.ac.uk})}  \and 
	 Simone Parisotto\thanks{Department of Applied Mathematics and Theoretical Physics,
	 	University of Cambridge, United Kingdom (\texttt{sp751@cam.ac.uk}, \texttt{cbs31@cam.ac.uk})} \and
	 Carola-Bibiane Sch\"onlieb\footnotemark[4]
 }
\numberwithin{equation}{section}
\newtheoremstyle{exampstyle}
{4pt} 
{4pt} 
{\itshape} 
{} 
{\bfseries} 
{.} 
{.5em} 
{} 
\theoremstyle{exampstyle}

\DeclareMathOperator*{\argmin}{argmin}

\DeclareMathOperator{\TV}{TV}

\DeclareMathOperator{\fGL}{GL_{\mathit{\varepsilon,\mu},\mathit{f}}}

\DeclarePairedDelimiter\ip{\langle}{\rangle_{\V}}
\newcommand{\be}{\begin{equation}}
\newcommand{\ee}{\end{equation}}
\newcommand{\bes}{\begin{equation*}}
\newcommand{\ees}{\end{equation*}}
\newcommand{\V}{\mathcal{V}}
\newcommand{\bigO}{\mathcal{O}}
\newtheorem{thm}{Theorem}[section]
\newtheorem{mydef}[thm]{Definition}
\newtheorem{example}[thm]{Example}

\newtheorem{ass}[thm]{Assumption}
\newtheorem{lemma}[thm]{Lemma}
\newtheorem{nb}[thm]{Note}
\newtheorem{prob}[thm]{Problem}

\begin{document}
	\maketitle
\begin{abstract}
 Practical image segmentation tasks concern images which must be reconstructed from noisy, distorted, and/or incomplete observations. A recent approach for solving such tasks is to perform this reconstruction jointly with the segmentation, using each to guide the other. However, this work has so far employed relatively simple segmentation methods, such as the Chan--Vese algorithm. In this paper, we present a method for joint reconstruction-segmentation using graph-based segmentation methods, which have been seeing increasing recent interest. Complications arise due to the large size of the matrices involved, and we show how these complications can be managed. We then analyse the convergence properties of our scheme.  Finally, we apply this scheme to distorted versions of ``two cows'' images familiar from previous graph-based segmentation literature, first to a highly noised version and second to a blurred version, achieving highly accurate segmentations in both cases. We compare these results to those obtained by sequential reconstruction-segmentation approaches, finding that our method competes with, or even outperforms, those approaches in terms of reconstruction and segmentation accuracy. 
 \vspace{1em}
 

\vspace{1em}
\end{abstract}	
\section{Introduction}

Two tasks which lie at the heart of many applications in image processing are \emph{image reconstruction}---the task of reconstructing an image from noisy, distorted, and/or incomplete observations---and \emph{image segmentation}, the task of separating an image into its ``important'' parts. But in practice, the latter task is not independent of the former: when we seek to segment an image, we will not typically have access to the true image, but rather must reconstruct the image from imperfect observations. That is, in practice a segmentation task is often a \emph{reconstruction}-segmentation task. 

Over the last decade, the framework of ``PDEs on graphs'' has yielded highly effective techniques for image segmentation. In this paper, we will exhibit a technique for reconstruction-segmentation within this framework. In particular, we will incorporate into this framework the method of \emph{joint reconstruction-segmentation}, which is an approach that performs the reconstruction and segmentation together, using each to guide the other, with the goal of improving the quality of the segmentation compared to performing the tasks in sequence. Previous implementations of this approach have employed relatively simple segmentation techniques. The key contribution of this paper will be to show how the more sophisticated graph-PDE-based segmentation techniques can be employed in joint reconstruction-segmentation. 


\subsection{Image reconstruction background}
The general setting for image reconstruction is that one has some observations $y$ of an image $x^*$, which are related via 
\be\label{invP}
y = \mathcal{T}(x^*) + e
\ee  
where $\mathcal{T}$ is the \emph{forward model}, typically a linear map, 
and $e$ is an error term (e.g. a Gaussian random variable). 
Solving \cref{invP} for $x^*$---given $y$, $\mathcal{T}$, and the distribution of $e$---is in general an ill-posed problem. A key approach to solving \cref{invP}, pioneered by Tikhonov \cite{Tikh} and Phillips \cite{phillips1962}, has been to solve the variational problem
\be\label{Tikhvar}
\argmin_x \mathcal{R}(x) + D(\mathcal{T}(x),y) 
\ee
where $\mathcal{R}$ is a \emph{regulariser}, encoding \emph{a priori} information about $x^*$, and $D$ enforces fidelity to the observations and encodes information about $e$. 

\subsection{Image segmentation background}\label{sec:segbg}
	One of the most celebrated methods for image segmentation is that of Mumford and Shah \cite{MS}. This method segments an image $x:\Omega \to \mathbb{R}$ by constructing a piecewise smooth $\tilde x \approx x$ and a set of contours $\Gamma$ (the boundaries of the segments) minimising a given segmentation energy, namely the Mumford--Shah functional 
	\be\label{eq:MS}\operatorname{MS}(\tilde x ,\Gamma) :=\int_{\Omega\setminus\Gamma} |\nabla \tilde x|^2 \; d\mu + \alpha \int_\Omega (x - \tilde x)^2 \; dx +  \beta |\Gamma|,\ee 
	 where $\mu$ is the Lebesgue measure. As $\operatorname{MS}$ is difficult to minimise in full generality, Chan and Vese \cite{CV} devised a method where $\tilde x$ is restricted to being piecewise constant.\footnote{Chan and Vese also add an extra energy term proportional to the area ``inside'' $\Gamma$.} This simplifies \cref{eq:MS} to an energy which can be minimised via level-set methods. Some key drawbacks of these methods are that they: can be computationally expensive, as one must solve a PDE; can be hard to initialise \cite{gao2005image}; can perform poorly if the image has inhomogeneities   \cite{zhang2010active}; and are constrained by the image geometry, and so less able to detect large-scale non-local structures. 
	
	These Mumford--Shah methods are related to Ginzburg--Landau methods, see e.g. \cite{ET}, because the Ginzburg--Landau functional $\Gamma$-converges to total variation \cite{MM}. 
	Because of this  $\Gamma$-convergence (which also holds on graphs \cite{vGB}), 
	Bertozzi and Flenner \cite{BF} were inspired to develop a segmentation method based on minimising the graph Ginzburg--Landau functional using the graph Allen--Cahn gradient flow. Soon after, 
	Merkurjev, Kosti\'c, and Bertozzi \cite{MKB} introduced an alternative method using a graph Merriman--Bence--Osher (MBO) scheme. 
	These ``PDEs on graphs'' methods have received considerable attention, both theoretical, see e.g. \cite{vGGOB,LB2016,BM},  and in applications, see e.g. \cite{mcBertozzi,Birdspot,BayesianGraphs,Bertozzi2021,PoissonMBO,Miller}. 
	In previous work by some of the authors, Budd and Van Gennip \cite{Budd1} showed that the graph MBO scheme is a special case of a semi-discrete implicit Euler (SDIE) scheme for graph Allen--Cahn flow, and  Budd, Van Gennip, and Latz \cite{Budd3} investigated the use of this SDIE scheme for image segmentation and developed refinements to earlier methods that resulted in improved segmentation accuracy.

\subsection{Joint reconstruction-segmentation background}
Reconstruction-segmentation was traditionally approached \emph{sequentially}: first reconstruct the image, then segment the reconstructed image. The key drawback of this method is that the reconstruction ignores any segmentation-relevant information. 
At the other extreme is the \emph{end-to-end} approach: first collect training data $\{(y_n,u_n)\}$ of pairs of observations and corresponding segmentations, then use this data to learn (e.g., via deep learning) 
a map that sends $y$ to $u$. However, this forgoes explicitly reconstructing $x^*$, can require a lot of training data, and the map can be a ``black box'' (i.e., it may be hard to explain its segmentation or prove theoretical guarantees).   

Joint reconstruction-segmentation (a.k.a. simultaneous reconstruction and segmentation) lies between these extremes, seeking to perform the reconstruction and segmentation simultaneously, using each to guide the other. It was first proposed by Ramlau and Ring \cite{RR2007} for CT imaging, with related (but extremely varied) methods later developed for other medical imaging tasks (for an overview, see \cite[\S2.4]{Corona2019}). An extensive theoretical overview of \emph{task-adapted reconstruction} was developed in Adler \emph{et al.} \cite{Adler2018}, which found that joint reconstruction-segmentation produced more accurate segmentations than both the sequential and end-to-end approaches. These methods were enhanced in Corona \emph{et al.} \cite{Corona2019} using Bregman methods, and a number of theoretical guarantees were proved about this enhanced scheme. However, these approaches have mostly relied on Mumford--Shah or Chan--Vese methods for the segmentation.

\subsection{Contributions and outline}
The primary contribution of this work will be a joint reconstruction-segmentation method based around the joint minimisation problem 
\[
\min_{x\in\mathbb{R}^{N\times \ell},u\in \V} \mathcal{R}(x) + \alpha \|\mathcal{T}(x) - y\|_F^2 + \beta 
\operatorname{GL}_{\varepsilon,\mu, f}(u,\Omega(\mathcal{F}(x),z_{d})),
\] 
where $x$ is the reconstruction, $u$ is the segmentation, the first two terms describe a reconstruction energy as in \cref{Tikhvar}, and the final term is a segmentation energy using the graph Ginzburg--Landau energy. The use of this energy is motivated by the success of the graph Ginzburg--Landau-based segmentation methods described in \Cref{sec:segbg}. These objects, and other groundwork required for this paper, will all be defined in \Cref{sec:gwork}.
In particular, in this paper:
\begin{enumerate} [i.]
	\item We will present an iterative scheme for solving this minimisation problem, which alternately updates the candidate reconstruction and the candidate segmentation (\Cref{sec:jrs}).
	\item We will devise algorithms for computing the steps of this iterative scheme (\Cref{xupdatesec,uupdatesec,sec:pipeline}). We compute the reconstruction update by linearising the corresponding variational problem (\Cref{xupdatesec}). We compute the segmentation update via the SDIE scheme (\Cref{uupdatesec}).
	\item We will demonstrate the convergence of this iterative scheme to critical points of the joint minimisation problem (\Cref{sec:convergence}).
	\item We will apply this scheme to highly-noised and to blurred versions of the ``two cows'' image familiar from \cite{BF,MKB,Budd3,Buddthesis} (\Cref{sec:applications}). Our scheme will exhibit very accurate segmentations which compete with or outperform sequential reconstruction-segmentation approaches.
\end{enumerate}




\section{Various groundwork}\label{sec:gwork}
\subsection{Framework for analysis on graphs}
 We begin by giving a framework for analysis on graphs, abridging Budd \cite[\S2]{Buddthesis}, which itself is abridging Van Gennip \emph{et al.} {\cite{vGGOB}}. 
 
 Let $(V,E,\omega)$ be a finite, undirected, weighted, and connected graph with neither multi-edges nor self-loops. The finite set $V$ is the \emph{vertex set}, $E\subseteq V^2$ is the \emph{edge set} (with $ij\in E$ if and only if $ji\in E$ for all $i,j\in V$), and $\{\omega_{ij}\}_{i,j\in V}$ are the \emph{weights}, with $\omega_{ij}\geq 0$, $\omega_{ij} = \omega_{ji}$, and $\omega_{ii}=0$, and $\omega_{ij} > 0$ if and only if $ij\in E$. We define function spaces 
 \begin{align*}
 	&\V := \left\{ u: V\rightarrow\mathbb{R} \right\} , &\V_{X} := \left\{ u: V\rightarrow X \right\}&,  &\mathcal{E} := \left\{ \varphi: E\rightarrow\mathbb{R} \right\},&
 \end{align*}
if $X\subseteq \mathbb{R}$.
For a parameter $r\in [0,1]$, and writing $d_i:=\sum_j \omega_{ij}$ for the \emph{degree} of vertex $i\in V$,  we define inner products on $\V$ and $\mathcal{E}$ (and hence inner product norms $\|\cdot\|_{\V}$ and $\|\cdot\|_{\mathcal{E}}$):
\begin{align*}\label{lapdef}
	&\ip{u,v} := \sum_{i\in V} u_i v_i d_i^r, &\langle\varphi,\phi\rangle_{\mathcal{E}}:=\frac{1}{2}\sum_{i,j\in V} \varphi_{ij} \phi_{ij}\omega_{ij}.&
\end{align*} 
Next, we introduce the graph variants of the gradient and Laplacian operators:
\[
	(\nabla u)_{ij}:=\begin{cases}u_j -u_i, & ij\in E,\\ 0, &\text{otherwise}, \end{cases} \qquad \text{and} \qquad (\Delta u)_i:=d_i^{-r}\sum_{j\in V}\omega_{ij}(u_i-u_j),
\]
where the graph Laplacian $\Delta$ is positive semi-definite and self-adjoint with respect to $\V$. As shown in \cite{vGGOB}, these operators are related via
$\ip{u,\Delta v} = \langle \nabla u, \nabla v \rangle_{\mathcal{E}}$.
We can interpret $\Delta$ as a matrix. Define $D := \operatorname{diag}(d)$ (i.e. $D_{ii}:=d_i$, and $D_{ij}:=0$ otherwise) to be the \emph{degree matrix}. Then writing $\omega$ for the matrix of weights $\omega_{ij}$ we get 
\[
\Delta:= D^{-r}(D-\omega).
\] The choice of $r$ is important. For $r=0$, $\Delta = D -\omega$ is the standard \emph{unnormalised (or combinatorial) Laplacian}. For $r=1$, $\Delta=I-D^{-1}\omega$ (where $I$ is the identity matrix) is  the \emph{random walk Laplacian}. There is also an important Laplacian not covered by this definition:  the \emph{symmetric normalised Laplacian} $\Delta_s := I-D^{-1/2}\omega D^{-1/2}$.  For image segmentation it is important to use a normalised Laplacian, see {\cite[\S2.3]{BF}}, so we shall henceforth take $r = 1$. 


\subsection{The graph Ginzburg--Landau functional}\label{subsec:GL}
In this paper, we shall use graph-based segmentation methods based on minimising the graph Ginzburg--Landau functional.	The basic form of this functional is
	\begin{equation*}
		\operatorname{GL}_{\varepsilon}(u) := \frac{1}{2}\left|\left|\nabla u\right|\right|^2_{\mathcal{E}} +\frac{1}{\varepsilon}\left\langle W\circ u,\mathbf{1}\right \rangle_{\V},
	\end{equation*}
	where $W$ is a double-well potential and $\varepsilon >0$ is a parameter. In particular, following Budd {\cite{Buddthesis}} we shall be taking $W$ to be the \emph{double-obstacle potential}: 
	\begin{equation*}
		W(x) := \begin{cases}
			\frac{1}{2}x(1-x), & \text{for } 0 \leq x \leq 1, \\
			\infty, & \text{otherwise.}  \end{cases}
	\end{equation*}
	 Furthermore, we define the \emph{graph Ginzburg\textendash Landau functional with fidelity} by 
	\bes 
	\operatorname{GL}_{\varepsilon,\mu,f}(u) := \frac{1}{2}\left|\left|\nabla u\right|\right|^2_{\mathcal{E}} +\frac{1}{\varepsilon}\left\langle W\circ u,\mathbf{1} \right \rangle_{\V} + \frac{1}{2}\ip{u-f, M(u-f)},
	\ees 
	where $M:=\operatorname{diag}(\mu)$ for $\mu\in\V_{[0,\infty)}$ the \emph{fidelity parameter} and $f \in \V_{[0,1]}$ is the \emph{reference}. We define $Z:=\operatorname{supp}(\mu)$, which we call the \emph{reference data}. 
	Note that $\mu_i$ paramaterises the strength of the fidelity to the reference at vertex $i$. We may assume without loss of generality that $f$ is supported on $Z$.

It is worth briefly describing why minimising $\fGL$ is a good way to segment an image. 
The first term penalises the segmentation $u$ if two vertices with a high edge weight are in different segments, encouraging the segmentation to group similar vertices together. The second term wants the segmentation to be binary. The third term penalises $u$ for disagreeing with an \emph{a priori} segmentation, propagating those \emph{a priori} labels to the rest of the vertices.

It will be useful for our joint reconstruction-segmentation scheme to redefine $\operatorname{GL}_{\varepsilon,\mu,f}$ as a function of both $u$ and $\omega$. A simple calculation gives that
\begin{equation*}
		\operatorname{GL}_{\varepsilon,\mu, f}(u,\omega) 
		= \frac{1}{2}\sum_{i,j\in V} \omega_{ij}(u_i-u_j)^2 +\frac{1}{2\varepsilon}\sum_{i,j\in V}  \omega_{ij}(W(u_i)+W(u_j)) 
		+ \frac{1}{4}\sum_{i,j\in V}\omega_{ij}\left(\mu_i(u_i-f_i)^2+\mu_j(u_j-f_j)^2  \right). 
\end{equation*}
Note that this is linear in $\omega$. We can therefore define $G_{\varepsilon,\mu,f}: \V \to \mathbb{R}^{V\times V}$ as
\begin{equation*}
	(G_{\varepsilon,\mu,f}(u))_{ij} = 
	\frac12(u_i-u_j)^2 + \frac{1}{2\varepsilon}(W(u_i)+W(u_j)) + \frac14\left(\mu_i (u_i-f_i)^2+\mu_j (u_j-f_j)^2\right),
\end{equation*}
such that
\[
\operatorname{GL}_{\varepsilon,\mu, f}(u,\omega) = \operatorname{tr}(G_{\varepsilon,\mu,f}(u)^T\omega)=:\langle G_{\varepsilon,\mu,f}(u), \omega\rangle_{F},
\]
where $\langle\cdot,\cdot\rangle_{F}$ denotes the Frobenius inner product.  
Furthermore, note that if $v_i := \frac12 u_i^2 + \frac{1}{2\varepsilon} W(u_i) + \frac14 \mu_i(u_i - f_i)^2$, then 
$
G_{\varepsilon,\mu,f}(u) = -uu^T + v\mathbf{1}^T + \mathbf{1}v^T$.

\subsection{Turning an image into a graph}\label{imtograph}

To represent an image as a graph, we first let our vertex set $V$ be the set of pixels in the image, and consider the image as a function $x : V\rightarrow \mathbb{R}^\ell$, where $\ell$ depends on whether the image is greyscale, RGB, or hyperspectral etc. 
To build our graph, we construct \emph{feature vectors}  $\mathcal{F}(x) =: z:V\rightarrow\mathbb{R}^q$ where $\mathcal{F}$ is the \emph{feature map} (which we shall assume to be linear). The philosophy behind these vectors is that vertices which are ``similar'' should have nearby feature vectors. What ``similar'' means is application-specific, e.g. \cite{VZ} incorporates texture into the features, \cite{BF,Birdspot} give other options,  and there has been recent interest in deep learning methods for constructing features, see e.g. \cite{AAR,Miller}.
Next, we construct the weights on the graph by defining $\omega_{ij}$ to be given by some \emph{similarity function} evaluated on $z_i$ and $z_j$. There are a number of standard choices for the similarity function, see e.g. \cite[\S2.2]{BF}. For our choices for feature map and similarity measure see \Cref{fmap} and \cref{Omega}, respectively.


\subsection{The Nystr\"om extension}\label{sec:Nys}

A key practical challenge is that $V$ is usually very large, and hence matrices such as the weight matrix $\omega\in\mathbb{R}^{V\times V}$ and $\Delta$ are much too large to store in memory. Instead, we shall compress these matrices using a technique called the Nystr\"om extension, first introduced in Nystr\"om \cite{Nys1} and developed for matrices in Fowlkes \emph{et al.} \cite{FBCM}. 
Consider an $N\times N$ symmetric matrix $A$, written in block form 
\[ A =
\begin{pmatrix}
	A_{XX} & A_{X X^c} \\ A_{X^c X} & A_{X^c X^c}
\end{pmatrix},
\]
where $X$ is the \emph{interpolation set}, with $|X|=:K\ll N$. Let $A_{XX} = U_X \Lambda U_X^T$, and $u_X^i$ be a column eigenvector of $U_X$ with eigenvalue $\lambda_i$. The idea of the Nystr\"om extension is to extend this eigenvector to a vector $((u_X^i)^T \:\:  (u^i_{X^c})^T)^T$ which is defined on all of $V$, using a quadrature rule. That is, 
$u_{X^c}^i$ is defined by 
$
\lambda_i  u^i_{X^c} = 
A_{X^cX}u_X^i
$,
which can be observed to resemble a quadrature rule for the eigenvalue problem $Au = \lambda_i u$.
Let $U_{X^c} := \begin{pmatrix}
	u^1_{X^c} & \cdots & u^K_{X^c}
\end{pmatrix}$. Then
$
U_{X^c}\Lambda = A_{{X^c}X} U_X 
$
and so (assuming that $A^{-1}_{XX}$ exists) $U_{X^c} = A_{{X^c}X}U_X \Lambda^{-1}$.
Finally, 
\begin{equation}
	\label{Nys}
	\begin{split}
		A \approx \begin{pmatrix} U_X \\  U_{X^c} \end{pmatrix} \Lambda \begin{pmatrix} U^T_X &  U^T_{X^c} 	 \end{pmatrix}
		= 
		\begin{pmatrix}
			A_{XX} & A^T_{{X^c}X} \\ A_{{X^c}X} & A_{{X^c}X}A_{XX}^{-1}A_{{X^c}X}^T
		\end{pmatrix}
	=
	 \begin{pmatrix}A_{XX} \\  A_{{X^c}X} \end{pmatrix} A_{XX}^{-1} \begin{pmatrix} A_{XX} &  A_{{X^c}X}^T 	 \end{pmatrix},
	\end{split}
\end{equation} 
where in the first equality we used that  $U_{X^c} = A_{{X^c}X}U_X \Lambda^{-1}$. 
The upshot of \cref{Nys} is that we only need to store and calculate with $A_{XX}$ and $A_{{X^c}X}$, which are much smaller than $A$. 
Also, \cref{Nys} yields an efficient way to approximate matrix-vector products $Av$.

\section{A joint-reconstruction-segmentation scheme on graphs}\label{sec:jrs}
\subsection{Set-up}\label{sec:setup}
We will begin by formally stating our reconstruction-segmentation task. 
\begin{prob}\label{RSprob}
	Let $x^*:Y\to\mathbb{R}^\ell$ be the image to be reconstructed and segmented. Let $y = \mathcal{T}(x^*) + e$ be \emph{observed data} where the \emph{forward model} $\mathcal{T}$ is differentiable and $e$ is a random variable describing observation error. Let $x_d:Z\to \mathbb{R}^\ell$ be an already reconstructed and segmented \emph{reference image} with \emph{a priori} segmentation $f:Z\to\{0,1\}$. Here $Y$ and $Z$ are disjoint finite sets. Given $y$, $\mathcal{T}$, $x_d$, and $f$, reconstruct $x\approx x^*$ and find $u:Y\cup Z\to\{0,1\}$ such that $u|_Y$ segments $x$ and $u|_Z$ is close to $f$. 
\end{prob}

Next, we must incorporate this into a graph framework, following \Cref{imtograph}. Let $V:= Y \cup Z$ be the vertex set of our graph, and let the edge set $E$ be given by $E:= \{ij \mid i,j\in V, i \neq j\}$. Let $N:=|Y|$ and $N_d := |Z|$. To encode the candidate reconstruction $x:Y\to\mathbb{R}^\ell$ and the reference image $x_d$ in the weights on this graph, we define feature maps $\mathcal{F}$ and $\mathcal{F}_d$, 
and feature vectors $z:Y\rightarrow \mathbb{R}^q$ and $z_{d}:Z\rightarrow \mathbb{R}^q$ by $z := \mathcal{F}(x)$ and $z_{d} := \mathcal{F}_d(x_d)$. Since $x_d$ and $\mathcal{F}_d$ are given, we hereafter treat $z_d$ as given. We then define the edge weights via $\omega = \Omega(z,z_d)$, where $\Omega(z,z_d)$ is given by (for $\mathbf{z}:=(z,z_d)$)
\be \label{Omega}
\Omega_{ij}(z,z_d) := e^{-\frac{\|\mathbf{z}_i - \mathbf{z}_j\|_F^2}{q\sigma^2}},
\ee
with $\|\cdot \|_F$ denoting the Frobenius norm.\footnote{This choice of edge weight function is not arbitrary. The fact that it has a particularly well-structured derivative will be used in \Cref{sec:g} and the fact that it is analytic will be used in \Cref{sec:convergence}. } The $q$ in the denominator averages over the $q$ components of $\mathbf{z}$ so that parameter choices for $\sigma$ generalise better. 

\begin{nb}
	The feature vectors $z$ and $z_d$ are defined so that $z$ does not depend on $x_d$ and $z_d$ does not depend on $x$. This is a simplification, since $x$ and $x_d$ might be different parts of the same image and hence one might want $z$ to partially depend on $x_d$. However, this simplification greatly aids in the following analysis, and in computation, as it means that the edge weights between vertices of $Z$ can be considered fixed and given.
\end{nb}
\subsection{The joint reconstruction-segmentation scheme}

To solve \Cref{RSprob}, we will employ a variational approach. We will consider our candidate reconstructions $x$ and segmentations $u$ to be candidate solutions to the following joint minimisation problem: 
\be\label{RSscheme1}
\min_{x\in\mathbb{R}^{N\times \ell}, u\in \V} \mathcal{R}(x) + \alpha \|\mathcal{T}(x) - y\|_F^2 + \beta 
\operatorname{GL}_{\varepsilon,\mu, f}(u,\Omega(\mathcal{F}(x),z_{d})),
\ee
where $\mathcal{R}$ is a convex regulariser, which following \Cref{primaldual} we shall assume can be written as $\mathcal{R}(x) = R(\mathcal{K}(x))$ for $\mathcal{K}$ a linear map and $R$ convex and  lower semicontinuous (l.s.c.) with convex conjugate $R^*$\footnote{\label{foot:R*}That is, $R^*(p):= \sup_{x'} \langle p,x'\rangle - R(x')$.} proper, convex, l.s.c., and non-negative. The first two terms in the objective functional are a standard Tikhonov reconstruction energy as in \cref{Tikhvar}, and the final Ginzburg--Landau term is the segmentation energy. As this problem (and related variational problems considered in this paper) are non-convex, by ``solving'' we will mean finding adequate local minimisers.  

To avoid needing to solve the difficult problem \cref{RSscheme1} directly, we will use the following alternating iterative scheme to approach solutions (where $\alpha,\beta,\eta_n,\nu_n$ are parameters):
\begin{subequations}\label{RSscheme2}
	\begin{align}
		\label{xupdate2}x_{n+1} &= \argmin_{x\in\mathbb{R}^{N\times \ell}}  \mathcal{R}(x) + \alpha \|\mathcal{T}(x) - y\|_F^2 + \beta\operatorname{GL}_{\varepsilon,\mu, f}(u_n,\Omega(\mathcal{F}(x),z_d)) 
		+\eta_n\|x-x_n\|_F^2,&& \\
		\label{uupdate2}u_{n+1} &= \argmin_{u\in\V}  \beta\operatorname{GL}_{\varepsilon,\mu, f}(u,\Omega(\mathcal{F}(x_{n+1}),z_d)) + \nu_n\|u|_Y-u_n|_Y\|_{\V}^2.&&
	\end{align}
\end{subequations}
We can understand this scheme intuitively as iterating the following steps:
\begin{enumerate}[I.]
	\item Given the current segmentation, update the reconstruction using the segmentation energy as an extra regulariser and the previous reconstruction as a momentum term.
	\item Given the current reconstruction, update the segmentation using the previous segmentation of the image to be reconstructed as a momentum term.
\end{enumerate}

\subsection{Initialisation}

The simplest initial reconstruction $x_0$ would be $x_0 := \mathcal{T}^+(y)$, where $\mathcal{T}^+$ is the (Moore--Penrose) pseudoinverse of $\mathcal{T}$ (see \cite[\S 5.5.2]{GVL}). However, in practice $\mathcal{T}^+(y)$ can be too poorly structured to give a good initial segmentation. Also, the pseudoinverse can be highly unstable \cite[\S 5.5.3]{GVL} and does not exist for non-linear $\mathcal{T}$. Thus, we favour initialising by applying a cheap and better behaved reconstruction method to $y$. 
The initial segmentation $u_0$ is constructed by segmenting $x_0$ via the SDIE methods to be described in \Cref{uupdatesec}.

\section{Solving \cref{xupdate2}}\label{xupdatesec}
We now describe how we compute (approximate) solutions to \cref{xupdate2}. This minimisation problem is highly computationally challenging, so we will simplify it by linearising \cref{xupdate2}. This reduces the problem to one which can be solved by standard methods. 
\subsection{Linearising \cref{xupdate2}}
The challenging term in \cref{xupdate2} is the Ginzburg--Landau energy term. Recall from \Cref{subsec:GL} that this can be written, 
\[
\beta\langle{G_n,\Omega(\mathcal{F}(x),z_d)}\rangle_F \simeq \underbrace{\beta\langle{(G_n)_{YY},\Omega_{YY}(\mathcal{F}(x))} \rangle_F}_{=:\mathscr{F}_1(\mathcal{F}(x))} +\underbrace{2\beta\langle{(G_n)_{YZ},\Omega_{YZ}(\mathcal{F}(x),z_d)} \rangle_F}_{=:\mathscr{F}_2(\mathcal{F}(x))},
\]
where for a matrix $A$, $A_{YZ}:= (A_{ij})_{i \in Y,j \in Z}$ and likewise for $A_{YY}$, and where
\be\label{Gn}
G_n := G_{\varepsilon,\mu,f}(u_n) = -u_nu_n^T + v_n\mathbf{1}^T + \mathbf{1}v_n^T,
\ee
for $v_n$ defined by $(v_n)_i:=\frac12 (u_n)_i^2 + \frac{1}{2\varepsilon}W((u_n)_i) + \frac14\mu_i((u_n)_i-f_i)^2$.
Let us assume that our candidate minimiser for \cref{xupdate2} is close to $x_n$ (this assumption will become more reasonable the higher the value of $\eta_n$ is). Then we can make the following approximation:
\begin{align*}
	\mathscr{F}_1(\mathcal{F}(x)) + \mathscr{F}_2(\mathcal{F}(x))
	&\approx \mathscr{F}_1(\mathcal{F}(x_n)) + \mathscr{F}_2(\mathcal{F}(x_n)) + \left\langle x - x_n, \nabla_x\mathscr{F}_1(\mathcal{F}(x_n)) + \nabla_x\mathscr{F}_2(\mathcal{F}(x_n)) \right\rangle \\
	&= \langle x, g_n \rangle + \text{constant terms (in $x$)},
\end{align*}
where $g_n:= \nabla_x\mathscr{F}_1(\mathcal{F}(x_n)) + \nabla_x\mathscr{F}_2(\mathcal{F}(x_n))$. We will describe how to compute $g_n$ in \Cref{sec:g}. Using this approximation, we can approximate \cref{xupdate2} by solving
\be \label{xupdate2a}
\argmin_{x\in \mathbb{R}^{N\times \ell}} \mathcal{R}(x) + \langle x,g_n \rangle_{F} + \alpha \|\mathcal{T}(x)-y\|_F^2 + \eta_n\|x-x_n\|_F^2. 
\ee
Defining $\tilde x_n := x_n - \frac{1}{2}\eta_n^{-1} g_n$, \cref{xupdate2a} is equivalent to 
\be \label{xupdate3}
\argmin_{x\in \mathbb{R}^{N\times \ell}}  \mathcal{R}(x) + \alpha \|\mathcal{T}(x)-y\|_F^2 + \eta_n\|x-\tilde x_n\|_F^2.
\ee
This is of the form of a standard variational image reconstruction problem \cref{Tikhvar}. To solve \cref{xupdate3}, we shall be employing an algorithm of Chambolle and Pock \cite{ChPock}, see \Cref{primaldual}.
\begin{nb}\label{note:linearT} Due to difficulties in employing the algorithms of \emph{\cite{ChPock}} for non-linear $\mathcal{T}$, we will henceforth take $\mathcal{T}$ to be linear (except in \Cref{sec:convergence}). The framework we describe in this paper is however applicable for general $\mathcal{T}$, so long as one is able to efficiently solve \cref{xupdate3} for that $\mathcal{T}$. 
\end{nb}

\subsection{The algorithm for \cref{xupdate2}}


We use \Cref{RSalg2} to approximately solve \cref{xupdate3}, thereby approximately solving \cref{xupdate2}.
\begin{algorithm}[h]
	\caption{Algorithm for solving the linearised \cref{xupdate2}.\label{RSalg2}}
	\begin{algorithmic}[1]
		\Function{ReconUpdate}{$x_n, u_n, y, \mathcal{T}, z_d, f, V, Y, Z, \mathcal{F}, q, \sigma, R, \mathcal{K}, W, \alpha, \beta, \eta_n, 
			K$}
			\\ \Comment{Computes $x_{n+1}$ solving \cref{xupdate3}. Note: assumes that $\mathcal{T}$ is linear }
		\State $v_n = \frac12(u_n)^2 + \frac{1}{2\varepsilon}W(u_n) + \frac14 \mu \odot (u_n - f)^2$ \Comment{Squaring elementwise} 
		\State $z_n = \mathcal{F}(x_n)$
		\State $w_1 = \texttt{CProd}(z_n,(z_n,z_d),u_n,v_n,\sigma,V,Y,Z,K) $ \Comment{See \Cref{Cz} }
		\State $w_2 = \texttt{CProd}(z_n,\mathbf{1}_V,u_n,v_n,\sigma,V,Y,Z,K) $
		\State $g_n = \frac{4\beta}{q\sigma^2}\mathcal{F}^*(w_1 - w_2 \odot z_n)$ 
		\State $\tilde x_n = x_n - \frac12 \eta^{-1}_n g_n $
		\State $\texttt{proxG}: (x,\delta t) \mapsto  \left((\delta t^{-1}+2\eta_n)I + 2\alpha \mathcal {T}^* \mathcal{T}\right)^{-1}\left(2\alpha\mathcal{T}^*(y)+2\eta_n\tilde x_n+x/\delta t\right) $ \Comment{See \cref{proxG}}
		\State $\texttt{proxRS}: (x,\delta t) \mapsto \operatorname{prox}_{\delta t R^*}(x)$ \Comment{Recall that $\mathcal{R}(x) = R(\mathcal{K}x)$}
		\State $x_{n+1} = \texttt{PrimalDual} (x_n,2\eta_n, \mathcal{K}, \mathcal{K}^*, \texttt{proxRS}, \texttt{proxG})$ \Comment{See \Cref{pdalg}}
		\State \textbf{return} $x_{n+1}$
		\EndFunction
	\end{algorithmic}
\end{algorithm}

\section{Solving \cref{uupdate2}}\label{uupdatesec}
After rescaling by $\beta^{-1}$, we rewrite the objective function in \cref{uupdate2} as
\begin{align*}
	&
	\operatorname{GL}_\varepsilon(u,\Omega(\mathcal{F}(x_{n+1}),z_d)) + \frac12\sum_{i\in V}d_i\mu_i(u_i -f_i)^2 +  \frac12 \frac{2\nu_n}{\beta}\|u|_Y-u_n|_Y\|_{\V}^2 \\
	&= \operatorname{GL}_\varepsilon(u,\Omega(\mathcal{F}(x_{n+1}),z_d)) + \frac12\sum_{i\in V}d_i\mu'_i(u_i -f'_i)^2 \\
	&= \operatorname{GL}_{\varepsilon,\mu',{f}'}(u,\Omega(\mathcal{F}(x_{n+1}),z_d)),
\end{align*}
where $\mu' := \mu +2\nu_n\beta^{-1}\chi_Y$ and $f' := f + u_n\odot\chi_Y$. We have used that $\mu|_Y = f|_Y = \mathbf{0}$. 

\subsection{The SDIE scheme for minimising $\operatorname{GL}_{\varepsilon,\mu',\tilde{f}'}$}\label{subsec:SDIE}
Following \cite{Budd3,Buddthesis}, in order to minimise $\operatorname{GL}_{\varepsilon,\mu',\tilde{f}'}$ we shall consider its \emph{Allen--Cahn gradient flow}:
\begin{align}
	\label{fACE2}
	&\varepsilon \frac{du}{dt}(t) + \varepsilon \Delta u(t) + \varepsilon M'(u(t) - f') +\frac{1}{2}\mathbf{1}-u(t)= \beta(t),  
\end{align}
where $M':=\operatorname{diag}(\mu')$ and $\beta(t)$ arises from the subdifferential of $W$, 
see \cite[\S 3.3.3]{Buddthesis} for details. 
Following \cite[\S 4]{Buddthesis}, we compute solutions to \cref{fACE2} via an SDIE 
scheme, defined by 
\be
\label{fSD}
\left(1- \frac{\tau}{\varepsilon} \right)u_{n+1} -\mathcal{S}_\tau u_n+\frac{\tau}{2\varepsilon}\mathbf{1} =\frac{\tau}{\varepsilon}\beta_{n+1},
\ee
for $\tau >0$ a time step, $\beta_{n+1} $ a subdifferential term (see \cite[Definition 4.1.1]{Buddthesis}), 
and $\mathcal{S}_\tau$ the solution operator for fidelity-forced graph diffusion, described by the following theorem. 
\begin{thm}[\text{See \cite[Theorem 3.2.6]{Buddthesis}}]\label{fdiffusethm}
	The \emph{fidelity-forced graph diffusion} of $u_0\in \V$ is
	\begin{align}\label{fdiffuse}
		&\frac{du}{dt}(t) = -\Delta u(t)-M'(u(t)-f'),  &u(0) = u_0.
	\end{align}
	For $t,x\in \mathbb{R}$, let $F_t(x):= (1-e^{-tx})/x$, and extend $F_t$ to (real) matrix inputs via its Taylor series.
	Then, for any given $u_0\in\V$, \cref{fdiffuse} has a unique solution, given by the map:
	\bes
	u(t) = \mathcal{S}_t u_0 := e^{-t(\Delta + M')}u_0 + F_t(\Delta + M') M' f'.
	\ees
\end{thm}
The solution to \cref{fSD} is then given by the following theorem. 
\begin{thm}[\text{See \cite[Theorem 4.2.1]{Buddthesis}}]\label{fSDsolnthm}
	For $\tau\in[0,\varepsilon)$, \cref{fSD} has unique solution 
\begin{subequations}\label{fSDsoln}
\begin{equation} 
	(u_{n+1})_i=\begin{cases}
		0, &\text{if } \left(\mathcal{S}_\tau u_n\right)_i < \frac{\tau}{2\varepsilon},
		\\
		\frac{1}{2} + \frac{\left( \mathcal{S}_\tau u_n\right)_i - 1/2}{1-\frac{\tau}{\varepsilon}}
		, &\text{if }\frac{\tau}{2\varepsilon}
		\leq\left(\mathcal{S}_\tau u_n\right)_i < 1-\frac{\tau}{2\varepsilon},
		\\
		1, &\text{if } \left(\mathcal{S}_\tau u_n\right)_i \geq 1-\frac{\tau}{2\varepsilon}.
	\end{cases}
\end{equation}
For $\tau = \varepsilon$, \cref{fSD} has solutions 
\be\label{fSDmbo}
(u_{n+1})_i \in \begin{cases}
	\{1\}, &(\mathcal{S}_{\tau} u_n)_i > 1/2,\\
	[0,1], &(\mathcal{S}_{\tau} u_n)_i = 1/2,\\
	\{0\}, &(\mathcal{S}_{\tau} u_n)_i < 1/2.
\end{cases} \ee
\end{subequations}
\end{thm}
\begin{nb}
The $\tau = \varepsilon$ special case described by \cref{fSDmbo} is the \emph{graph MBO scheme}, which has seen widespread use in image segmentation, pioneered by Merkurjev \emph{et al.} \emph{\cite{MKB}}. 
\end{nb}

We describe how to compute this SDIE scheme in \Cref{sec:computeSDIE}.

\subsection{The algorithm for \cref{uupdate2}}
We summarise the above as \Cref{uupdatealg}.
\begin{algorithm}[h]
	\caption{Algorithm for solving \cref{uupdate2} using the SDIE scheme. \label{uupdatealg}}
	\begin{algorithmic}[1]
		\Function{SegUpdate}{$u_{n},x_{n+1},z_d,\mathcal{F},q,\sigma,V,Y,Z,K,\delta,\beta,\nu_n,\tau,\varepsilon,\mu,f,k_s$} \Comment{Computes the minimiser  \\ \hfill of \cref{uupdate2} by computing an SDIE sequence as in \Cref{sec:computeSDIE}
		} 
		\State $(\mu',f') = (\mu + 2 \nu_n \beta^{-1} \chi_Y,f + u_n \odot \chi_Y)$
		\State $\omega: ij \mapsto \Omega_{ij}(z,z_d,q,\sigma)$ \Comment{Defined as in \cref{Omega}}
		\State $[U_1,\Lambda,U_2] = \texttt{Nystr\"omQR}(\omega,V,Z,K/2,K/2)$ \Comment{See \Cref{nysQR}}
		\State $ F: x \mapsto  \left(-U_1\Lambda (U_2^T x) -\mu' \odot (x - f') \right)$ 
		\State $\hat v = \texttt{ode\_solver}(F,[0,\tau],\mathbf{0})$ \Comment{Solves $\hat v'(t)=F(\hat v)$ on $[0,\tau]$, $\hat v(0) = \mathbf{0}$, e.g. as in \cite{MKB}} 
		\State $ b = \hat v(\tau)$ \Comment{$b:=F_\tau(\Delta + M')M'f'$}
		\State $\Sigma = I_K - \Lambda$ 
			\State $(a_1,a_2,a_3) = \left( \operatorname{exp}(-\tau/k_s (\mu'+\mathbf{1})),  \operatorname{exp}(\tau/k_s \operatorname{diag}(\Sigma))-\mathbf{1}_K, \operatorname{exp}(-\frac12 \tau/k_s (\mu'+\mathbf{1})) \right)$ \Comment{See \cref{eq:Strang}}
		\State $u^0 = \frac12 \chi_Y + f $ \Comment{As an example initial condition}
		\State $ m=0$ 
		\While{$\|u^m-u^{m-1}\|_2^2/\|u^m\|_2^2\geq\delta$}
		\State $v = u^m$ 
		\For{$r = 1$ to $k_s$}
		\State $v =a_1 \odot v + a_3 \odot \left( U_1\left( a_2 \odot \left(U_2^T\left(a_3 \odot v\right) \right)\right) \right)$ \Comment{Strang formula iteration \cref{eq:Strang}}
		\EndFor
		\State $v = v + b$ \Comment{Approximates $v = \mathcal{S}_\tau u^m$}
		\State $V_1 =\{i\in V\mid v_i \in [\tau/2\varepsilon, 1 - \tau/2\varepsilon)\}$
		\State $V_2 = \{i\in V\mid v_i \geq 1 - \tau/2\varepsilon\}$
		\State $u^{m+1}= (1-\tau/\varepsilon)^{-1}(v - \tau/2\varepsilon \mathbf{1})\odot\chi_{V_1} + \chi_{V_2}$ \Comment{
				Applies the \cref{fSDsoln} thresholding}
		\State $m = m +1$ 
		\EndWhile
		\State \textbf{return} $u^m$
		\EndFunction 
	\end{algorithmic}
\end{algorithm}
\section{The full pipeline} \label{sec:pipeline}
We summarise the full pipeline as \Cref{RSalg}. 

\begin{algorithm}[h]
	\caption{Graph-based joint reconstruction-segmentation algorithm using \cref{RSscheme2}.\label{RSalg}}
	\begin{algorithmic}[1]
		\Function{JointRecSeg}{$y, \mathcal{T}, z_d, f, V, Y, Z, \mathcal{F}, q, \sigma, R, \mathcal{K}, W, \alpha, \beta, \delta, \eta_n, \nu_n, \tau, \varepsilon, \mu, K, N, k_s $}\\\hfill\Comment{Computes the first $N$ iterates of \cref{RSscheme2}  }
		\State $ x_0 = \texttt{cheap\_reconstruction}(y,\mathcal{T})$ \Comment{Initial cheap reconstruction}
		\State \parbox[t]{\dimexpr\linewidth-\algorithmicindent-0.0\linewidth}{$ u_0 = \texttt{SegUpdate}(f,x_0,z_d,\mathcal{F},q,\sigma,V,Y,Z,K,\delta,1,0,\tau,\varepsilon,\mu,f,k_s)$ \\ $ $ \strut \Comment{Initial SDIE segmentation, see \Cref{uupdatealg}}}
		\For{$n = 0$ to $N-1$} \Comment{The iterations of \cref{RSscheme2}}
			\State \parbox[t]{\dimexpr\linewidth-\algorithmicindent-0.04\linewidth}{$x_{n+1} = \texttt{ReconUpdate}(x_n, u_n, y, \mathcal{T}, z_d, f, V, Y, Z, \mathcal{F}, q, \sigma, R, \mathcal{K}, W, \alpha, \beta, \eta_n, K)$ \\ $ $ \strut \Comment{Solves \cref{xupdate2}, see \Cref{RSalg2}}} 
			\State \parbox[t]{\dimexpr\linewidth-\algorithmicindent-0.04\linewidth}{$u_{n+1} = \texttt{SegUpdate}(u_{n},x_{n+1},z_d,\mathcal{F},q,\sigma,V,Y,Z,K,\delta,\beta,\nu_n,\tau,\varepsilon,\mu,f,k_s)$ 
			\Comment{Solves \cref{uupdate2}}}
			\EndFor
			\State \textbf{return} $\{x_n\}_{n=0}^N, \{u_n\}_{n=0}^N$
		\EndFunction
	\end{algorithmic}
\end{algorithm}
\section{Convergence analysis}\label{sec:convergence}

In this section, we will show that \cref{RSscheme2} converges to critical points of \cref{RSscheme1}, using the theory from Attouch \emph{et al.} \cite{attouch2010proximal}. From \Cref{note:linearT} we recall that in this section we do not require $\mathcal{T}$ to be linear. We first rewrite \cref{RSscheme1} abstractly:
\bes 
\min_{x \in \mathcal{X},u\in \V} \mathscr{F}(x) + \mathscr{G}(u,x) + \mathscr{H}(u) =: \mathcal{J}(u,x),
\ees
where $\mathcal{X}:=\mathbb{R}^{N\times \ell}$, \begin{align*} \mathscr{F}(x) := \mathcal{R}(x) + \alpha\|\mathcal{T}(x) - y\|_F^2, && \mathscr{G}(u,x):=\beta \langle -uu^T + \mathbf{1}v^T  + v\mathbf{1}^T,\Omega(\mathcal{F}(x),z_d)\rangle_F,\end{align*}
(where $v_i:= \frac12 u_i^2 + \frac{1}{4\varepsilon}u_i(1-u_i) + \frac14 \mu_i(u_i-f_i)^2$), and $\mathscr{H}(u) := 0$ if $u \in \V_{[0,1]}$ and $\mathscr{H}(u):= \infty$ otherwise.
\begin{nb} Both $\mathscr{F}$ and $\mathscr{H}$ are proper and l.s.c., and $\mathscr{G}$ is $C^\infty$ (indeed, analytic), so \emph{\cite[Assumption \mH]{attouch2010proximal}} is satisfied. Furthermore, $\mathscr{G}(u,x) + \mathscr{H}(u) = \fGL(u,\Omega(\mathcal{F}(x),z_d))$. 
	\end{nb}
Then \cref{RSscheme2} can be written as 
\begin{subequations}\label{RSscheme4}
	\begin{align}
		x_{n+1} &= \argmin_{x \in \mathcal{X}} \mathscr{F}(x) + \mathscr{G}(u_n,x) + \eta_n\|x- x_n\|_F^2,\\
		\label{subeq:uupdate} u_{n+1} &= \argmin_{u \in \V} \mathscr{G}(u,x_{n+1}) +  \mathscr{H}(u) + \nu_n \|u|_Y - u_n|_Y\|_{\V}^2,
	\end{align}
\end{subequations}
and our partially linearised iterative scheme can be written as 
\begin{subequations}\label{RSscheme5}
	\begin{align}
		\tilde x_n &= x_n - \frac{1}{2}\eta_n^{-1}\nabla_x\mathscr{G}(u_n,x_n),\\
		x_{n+1} &= \argmin_{x \in \mathcal{X}} \mathscr{F}(x) + \eta_n\|x-\tilde x_n\|_F^2,\\
		u_{n+1} &= \argmin_{u \in \V} \mathscr{G}(u,x_{n+1}) +  \mathscr{H}(u) + \nu_n \|u|_Y - u_n|_Y\|_{\V}^2.
	\end{align}
\end{subequations}
However, the presence of the {\it semi}-norm $\|\cdot|_Y \|_{\V}$ in \cref{subeq:uupdate} is an obstacle to the deployment of the theory from \cite{attouch2010proximal}. Hence, we will make an assumption that for all $n$, $u_n|_Z = f$. In practice, we have observed that this approximately holds, and furthermore the larger the value of $\mu$ the more closely this will hold. Thus, defining $\V^f:=\{u\in \V\mid u|_Z = f\}$ and $\V_{[0,1]}^f:=\{u\in \V_{[0,1]}\mid u|_Z = f\}$, under this assumption \cref{RSscheme4} becomes:
\begin{subequations}\label{RSscheme6}
	\begin{align}
		x_{n+1} &= \argmin_{x \in \mathcal{X}} \mathscr{F}(x) + \mathscr{G}(u_n,x) + \eta_n\|x- x_n\|_F^2,\\
		 u_{n+1} &= \argmin_{u \in \V} \mathscr{G}(u,x_{n+1}) +  \mathscr{H}^f(u) + \nu_n \|u - u_n\|_{\V}^2,
	\end{align}
\end{subequations} 
where $\mathscr{H}^f(u) := 0$ if $ u \in \V^f_{[0,1]}$ and $\mathscr{H}^f(u) :=\infty$ otherwise. Let $\mathcal{J}^f(u,x):=  \mathscr{F}(x) + \mathscr{G}(u,x) + \mathscr{H}^f(u)$, which is equal to $\mathcal{J}(u,x)$ if $u \in \V^f$.
We begin by making some key definitions.\footnote{With respect to the question of what these definitions are for, we ask the reader to bear with us, with the hope that within a page or two their purpose will become clearer.}
\begin{mydef}[Kurdyka--Łojasiewicz property]
	A proper l.s.c. function $g:\mathbb{R}^n \to (-\infty,\infty]$ has the \emph{Kurdyka--Łojasiewicz property} at $\bar z \in \operatorname{dom} \partial g$\footnote{We denote by $\operatorname{dom} g$ the set of $z$ such that $g(z)< \infty$, and by $\operatorname{dom} \partial g$ the set of $z\in \operatorname{dom} g$ such that the (limiting) subdifferential of $g$ at $z$, $\partial g(z)$ (defined in \cite[Definition 2.1]{attouch2010proximal}), is non-empty. } if there exist $\eta \in (0,\infty]$, a neighbourhood $U$ of $\bar z$, and a continuous concave function $\varphi:[0,\eta) \to [0,\infty)$, such that 
	\begin{itemize}
		\item $\varphi$ is $C^1$ with $\varphi(0)=0$ and $\varphi' >0$ on $(0,\eta)$, and
		\item for all $z \in U$ such that $g(\bar z) < g(z) < g(\bar z)+ \eta$, the  \emph{Kurdyka--Łojasiewicz inequality} holds:
		\[
		\varphi'(g(z) - g(\bar z))\operatorname{dist}(\mathbf{0},\partial g(z)) \geq 1. 
		\]
			\end{itemize}
		If $\varphi(s) := cs^{1-\theta}$ is a valid concave function for the above with $c > 0$ and $\theta \in [0,1)$, then we will say that $g$ has the  \emph{Kurdyka--Łojasiewicz property with exponent $\theta$} at $\bar z$. 
	\end{mydef}
\begin{mydef}[Semi-analyticity and sub-analyticity]
Following e.g. Łojasiewicz \emph{\cite{lojasiewicz1964}}, we define $A\subseteq \mathbb{R}^n$ to be a \emph{semi-analytic set} if for all $z^* \in \mathbb{R}^n$ there exists a neighbourhood $U$ containing $z^*$ and a finite collection of analytic functions $(a_{ij},b_{ij})$ such that 
	\[
	A \cap U = \bigcup_i \bigcap_j \{ z \in U \mid a_{ij}(z) = 0 \text{ and } b_{ij}(z)>0\}.
	\]
	Following Hironaka \emph{\cite{hironaka1973subanalytic}}, we define $A$ to be a \emph{sub-analytic set} if for all $z^* \in \mathbb{R}^n$ there exists a neighbourhood $U$ containing $z^*$, $m \in \mathbb{N}$, and a bounded semi-analytic set $B\subset \mathbb{R}^{n+m}$ such that  
		\[
	A \cap U =  \{ z \in \mathbb{R}^n \mid \exists y \in \mathbb{R}^m \text{ such that } (z,y) \in B \}.
	\]
	We define $g:\mathbb{R}^n \to (-\infty,\infty]$ to be a \emph{semi-analytic function} if its graph $\operatorname{Gr} g := \{(z,z')\in \mathbb{R}^n \times \mathbb{R} \mid z'=g(z)\}$ is a semi-analytic set, and $g$ to be a \emph{sub-analytic function} if its graph is a sub-analytic set. Both of these collections of sets are closed under elementary set operations.
	\end{mydef}
We collect some key results regarding these definitions in the following lemma. 
\begin{lemma}\label{lem:subany}\begin{enumerate}[i.]
		\item If $g$ is proper and l.s.c., and $\bar z\in \operatorname{dom} g$ with $\mathbf{0} \notin \partial g(\bar{z})$, then for all $\theta \in [0,1)$, $g$ has the {Kurdyka--Łojasiewicz property with exponent $\theta$} at $\bar z$.
		\item If $g$ is proper and sub-analytic, $\operatorname{dom} g$ is closed, and $g$ is continuous on its domain, then for all $\bar z \in \operatorname{dom} g$ with $\mathbf{0} \in \partial g(\bar z)$, there exists $\theta \in [0,1)$ such that $g$ has the {Kurdyka--Łojasiewicz property with exponent $\theta$} at $\bar z$.
		\item If $g:\mathbb{R}^n \to (-\infty,\infty]$ is sub-analytic and $h:\mathbb{R}^n \to \mathbb{R}$ is analytic, then $g+h$ is sub-analytic.
		\item If $g:\V^f \times \mathcal{X} \to (-\infty,\infty]$ is sub-analytic, then $h:(u,x) \mapsto g(u,x) + \mathscr{H}^f(u)$ is sub-analytic.
	\end{enumerate}
	\end{lemma}
\begin{proof}\begin{enumerate}[i.]
		\item Proved in Li and Pong \cite[Lemma 2.1]{KLcalculus}.
		\item Proved in Bolte \emph{et al.} \cite[Theorem 3.1]{Bolte2007}.
		\item Fix $(z^*,w^*) \in \mathbb{R}^n \times \mathbb{R}$. Since $g$ is sub-analytic, there exists $U$ containing $(z^*,w^*-h(z^*))$ and a bounded semi-analytic set $B\subset \mathbb{R}^{n+1+m}$ such that 
			\[
		\operatorname{Gr} g  \cap U =  \{ (z,w) \in \mathbb{R}^n \times \mathbb{R} \mid \exists y \in \mathbb{R}^m \text{ such that } (z,w,y) \in B \}.
		\]
		Since $h$ is continuous, there exists a neighbourhood $V$ containing $(z^*,w^*)$ such that for all $(z,w) \in \operatorname{Gr}( g +h) \cap V$, $(z,w-h(z))\in  \operatorname{Gr} g \cap U$, and hence $(z,w-h(z),y) \in B$ for some $y\in\mathbb{R}^m$. Let $B':= \{(z,w+h(z),y)\mid (z,w,y) \in B\}$. Since $h$ is analytic, $B'$ is a bounded semi-analytic set, and for all  $(z,w) \in \operatorname{Gr}( g +h) \cap V$, there exists $y\in \mathbb{R}^m$ such that $(z,w,y)\in B'$. It follows that $g+h$ is sub-analytic. 
		\item $\operatorname{Gr}h = \{(u,x,y) \mid u \in \V_{[0,1]}^f \text{ and } (u,x,y) \in \operatorname{Gr} g \} = (\V_{[0,1]}^f\times \mathcal{X}\times \mathbb{R}) \cap \operatorname{Gr} g$. It is simple to check that $\V_{[0,1]}^f$, $\mathcal{X}$, and $\mathbb{R}$ are semi-analytic, and hence sub-analytic, and therefore if $\operatorname{Gr} g$ is sub-analytic then so is $\operatorname{Gr} h$.
	\end{enumerate}
\end{proof}
\begin{ass}\label{ass:R}
	Suppose that $\mathcal{R}$ is sub-analytic, continuous on its domain, bounded below, and $\operatorname{dom}\mathcal{R}$ is closed. Suppose also that $\mathcal{T}$ is analytic and that $\mathscr{F}(x) \to \infty$ as $\|x\|_F \to \infty$. 
	\end{ass}
\begin{nb}\label{note:assumptionsRT}
Examples of $\mathcal{R}$ satisfying this assumption are: $\mathcal{R}(x):=\|A x\|_1$ (commonly used in compressed sensing, see Adcock and Hansen \emph{\cite{Hansen}}), where $A$ is any matrix, and $\mathcal{R}$ given by a feedforward neural network with a ReLU activation function (commonly used as regularisers, see e.g. Arridge \emph{et al.} \emph{\cite{Arridge2019}}); see \Cref{thm:app1} for proofs. 
Examples of $\mathscr{F}$ satisfying the assumption are when $\mathcal{T}$ is an invertible linear map or $\mathcal{R}$ is coercive.
\end{nb}
	\begin{thm}[\text{See \cite[Lemma 3.1]{attouch2010proximal}}] \label{thm:RSthm}
	If \Cref{ass:R} holds and, for some $0<a<b$ and all $n$, $\eta_n,\nu_n \in (a,b)$, then $(u_n,x_n)_{n \in \mathbb{N}}$ solving \cref{RSscheme6} are well-defined, and furthermore:
	\begin{enumerate}[i.]
		\item $\mathcal{J}^f(u_{n+1},x_{n+1}) \leq \mathcal{J}^f(u_{n},x_{n})$ with equality if and only if $x_{n+1} = x_n$ and $u_{n+1}=u_n$; 
		\item $\lim_{n \to \infty} \|x_{n+1}-x_n\|_F+\|u_{n+1}-u_n\|_{\V} = 0$ (indeed, the differences are square-summable);
		\item for all bounded subsequences $(u_{n'},x_{n'})$, $\operatorname{dist}(\mathbf{0},\partial\mathcal{J}^f(u_{n'},x_{n'})) \to 0$.
	\end{enumerate}
\end{thm}
\begin{proof}
	Note that, by \Cref{ass:R}, $\mathcal{J}^f(u,x)$ is bounded below, and that for all $x\in\mathcal{X}$, $\mathcal{J}^f(\cdot,x)$ is proper. Hence given the assumption on $\eta_n$ and $\nu_n$, \cite[Assumption \mHo]{attouch2010proximal} is satisfied, and therefore the result follows by \cite[Lemma 3.1]{attouch2010proximal}.
\end{proof}
\begin{lemma}\label{lem:KLlem} 
	 Let \Cref{ass:R} hold.
	 Then for all $\bar u\in\V^f_{[0,1]}$ and for all $\bar x \in \operatorname{dom} \mathcal{R}$, there exists $\theta \in [0,1)$ such that $\mathcal{J}^f$ has the Kurdyka--Łojasiewicz property with exponent $\theta$ at $(\bar u, \bar x)$. 
	\end{lemma}
\begin{proof} Note first that $\operatorname{dom} \mathcal{J}^f = \V^f_{[0,1]} \times \operatorname{dom}\mathcal{R} $, which is therefore closed by the assumption. Next, note that by the assumption and the continuity of $\mathscr{G}$ and $\mathscr{H}^f$ (in the latter case, on its domain), $ \mathcal{J}^f$ is continuous on its domain. Finally, since $\mathscr{G}$ is analytic, $\|\mathcal{T}(x)-y\|_F^2$ is analytic, and $\mathcal{R}$ is sub-analytic, it follows by \Cref{lem:subany}(iii-iv) that $  \mathcal{J}^f$ is sub-analytic.

	Let $\bar u \in \V_{[0,1]}^f$ and $\bar x \in \operatorname{dom} \mathcal{R}$. If $\mathbf{0}\in \partial \mathcal{J}^f(\bar u,\bar x)$, then by the above  \Cref{lem:subany}(ii) applies. Thus there exists $\theta \in [0,1)$ such that ${ \mathcal{J}^f}$ has the Kurdyka--Łojasiewicz property with exponent $\theta$ at $(\bar u,\bar x)$. If instead $\mathbf{0}\notin \partial \mathcal{J}^f(\bar u,\bar x)$, then by \Cref{lem:subany}(i)---as $\mathcal{J}^f$ is proper and l.s.c.--- for all $\theta \in [0,1)$, ${ \mathcal{J}^f}$ has the Kurdyka--Łojasiewicz property with exponent $\theta$ at $(\bar u,\bar x)$.

	\end{proof}
\begin{lemma} \label{lem:bdd}	Suppose that, for some $0<a<b$ and all $n$, $\eta_n,\nu_n \in (a,b)$, and that \Cref{ass:R} holds. Then for all $u_0\in \V^f_{[0,1]}$ and for all $x_0 \in \operatorname{dom} \mathcal{R} $, if $(u_n,x_n)_{n \in \mathbb{N}}$ is defined from $(u_0,x_0)$ by \cref{RSscheme6}, then $\|(u_n,x_n)\|:=\|u_n\|_{\V} + \|x_n\|_F$ is bounded.
\end{lemma}
\begin{proof} As $u_n \in \V_{[0,1]}^f$ for all $n$, it suffices to show that $\|x_n\|_F$ is bounded. For all $u\in \V$ and for all $x\in \mathcal{X}$, $\mathscr{G}(u,x)+ \mathscr{H}^f(u) \geq 0$, and, by \Cref{thm:RSthm}(i), $\mathcal{J}^f(u_n,x_n) \leq \mathcal{J}^f(u_0,x_0)$ for all $n$. By \Cref{ass:R}, $\mathscr{F}(x) \to \infty$ as $\|x\|_F \to \infty$; hence $\|x_n\|_F$ is bounded.
\end{proof}
	\begin{thm}[\text{See \cite[Theorems 3.2, 3.3, and 3.4]{attouch2010proximal}}]
		Suppose that: for some $0<a<b$ and all $n$, $\eta_n,\nu_n \in (a,b)$; $(u_n,x_n)_{n \in \mathbb{N}}$ is defined from $(u_0,x_0)$ by \cref{RSscheme6}; and \Cref{ass:R} holds.
		
		Then for all $u_0\in \V^f_{[0,1]}$ and all $x_0 \in \operatorname{dom} \mathcal{R}$, $(u_n,x_n)$ converges to a critical point of $\mathcal{J}^f$. 
		
		Furthermore, suppose that $\bar u \in \V^f_{[0,1]}$, $\bar x \in \operatorname{dom} \mathcal{R}$, and $(\bar u, \bar x)$ is a global minimum of $\mathcal{J}^f$. Then there exists a neighbourhood $U$ containing $(\bar u,\bar x)$ and $\eta > 0$ such that if $(u_0,x_0) \in U$ and $\min \mathcal{J}^f < \mathcal{J}^f(u_0,x_0) < \min \mathcal{J}^f + \eta$, then $(u_n,x_n) \to (u^*,x^*)$, a global minimiser of $\mathcal{J}^f$. 
		
		Finally, let $(u_n,x_n) \to (u_\infty,x_\infty)$ where  $u_\infty \in \V^f_{[0,1]}$ and $ x_\infty \in \operatorname{dom} \mathcal{R}$. If $\theta\in[0,1)$ is the Kurdyka--Łojasiewicz exponent of $\mathcal{J}^f$ at $(u_\infty,x_\infty)$ (which exists by \Cref{lem:KLlem}), then: 
		\begin{enumerate}[i.]
			\item If $\theta = 0$, then $(u_n,x_n)$ converges in finitely many steps.
			\item  If $\theta \in (0,\frac12]$, then there exist $c >0$ and $\zeta \in [0,1)$ such that $\|(u_n,x_n)-(u_\infty,x_\infty)\|\leq c \zeta^n$.
			\item If $\theta \in (\frac12,1)$, then there exists $c >0$ such that $\|(u_n,x_n)-(u_\infty,x_\infty)\|\leq c n^{-(1-\theta)/(2\theta-1)}$.
		\end{enumerate}
		\end{thm}
	\begin{proof}
		Given the suppositions: \cite[Assumptions \mH \:  and \mHo]{attouch2010proximal} hold; by \Cref{lem:KLlem}, for all  $\bar u \in \V^f_{[0,1]}$ and $\bar x \in \operatorname{dom} \mathcal{R}$ there exists $\theta \in [0,1)$ such that $\mathcal{J}^f$ has the Kurdyka--Łojasiewicz property with exponent $\theta$ at $(\bar u, \bar x)$; and by \Cref{lem:bdd}, $\|(u_n,x_n)\|$ does not tend to infinity. Therefore the results follow  immediately from  \cite[Theorems 3.2, 3.3, and 3.4]{attouch2010proximal}, respectively.
	\end{proof}
\begin{nb} The above convergence results concerned \cref{RSscheme6}, in which everything is solved without linearisation. In Bolte \emph{et al.} \emph{\cite{BoltePALM}}, similar convergence results as in \emph{\cite{attouch2010proximal}} were proved for fully linearised alternating schemes for energies with the Kurdyka--Łojasiewicz property. It is beyond the scope of this work to extend such results to \cref{RSscheme5}, which is only partially linearised, but it is the authors' belief that such an extension is likely to be straightforward.
	\end{nb}
\section{Applications}\label{sec:applications}

We will test our scheme on distorted versions of the ``two cows'' images familiar from \cite{MKB,BF,Budd3,Buddthesis}.\footnote{These images are taken from the Microsoft Research Cambridge Object Recognition Image Database, available at \url{https://www.microsoft.com/en-us/research/project/image-understanding/}; accessed 4 March 2022.} 

\begin{example}[Two cows]\label{ex_twocows} 
	We take the image in the top left of \Cref{Fig_twocows} as the reference data $Z$ and the segmentation in the bottom left as the reference labels $f$, which separate the cows from the background. 
	The image $x^*$ to be segmented is the top-right image of \Cref{Fig_twocows}, with the ground truth segmentation in the bottom right.
	Both images are RGB images of size $480 \times 640$ pixels. 
	\Cref{fig:twocowsseg} shows the result of segmenting $x^*$ via \Cref{uupdatealg}, as in \emph{\cite[\S5.3.3]{Buddthesis}}.
\end{example}

\begin{figure}
	\centering
	\includegraphics[scale=0.6]{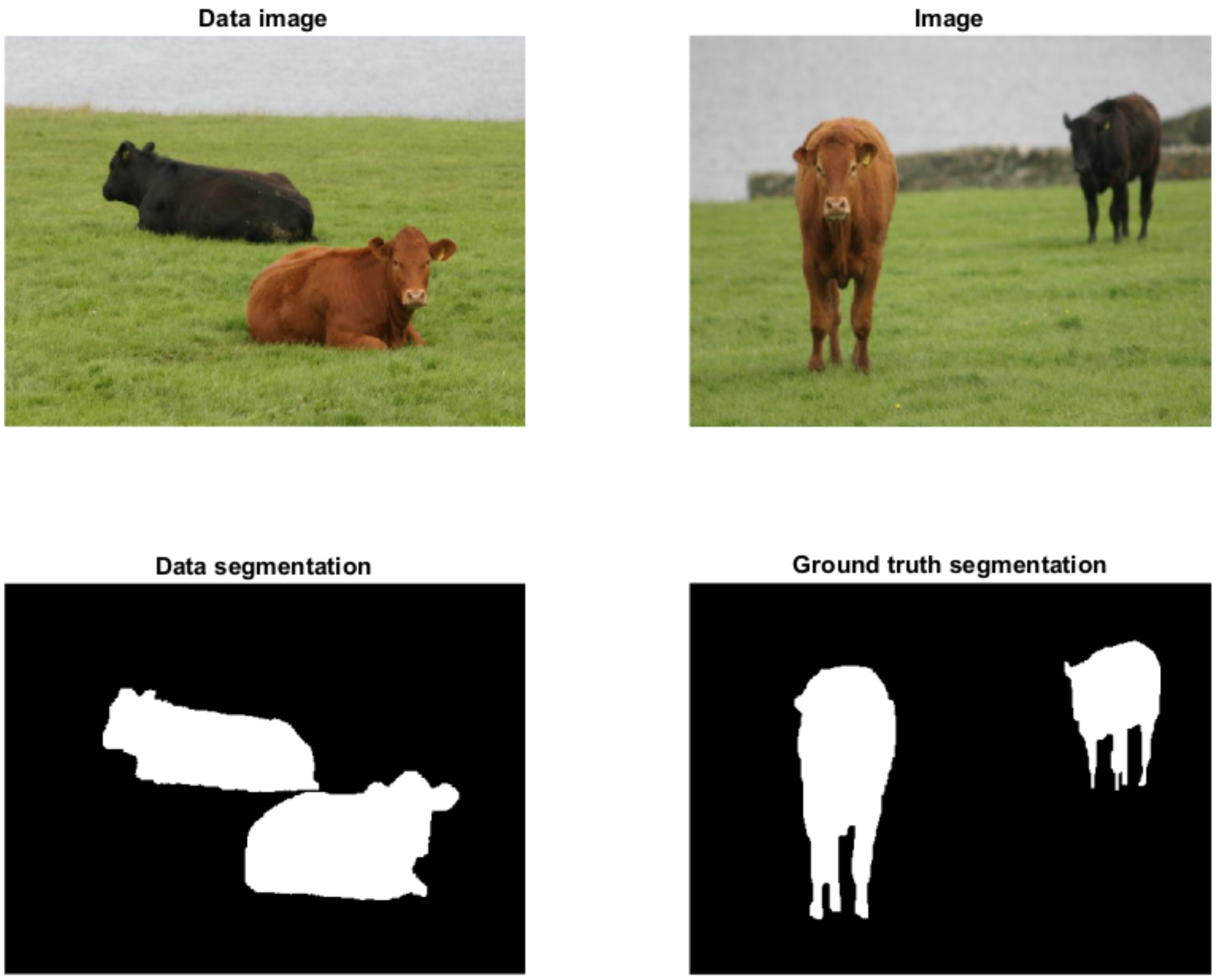}
	
	\caption{Two cows: the reference data image, the image to be segmented, the reference $f$ (which is a segmentation of the reference data image), and the ground truth segmentation associated to \Cref{ex_twocows}. Both segmentations were drawn by hand by the authors.} \label{Fig_twocows}
\end{figure}
\begin{figure}
	\centering
	\includegraphics[width=0.5\textwidth]{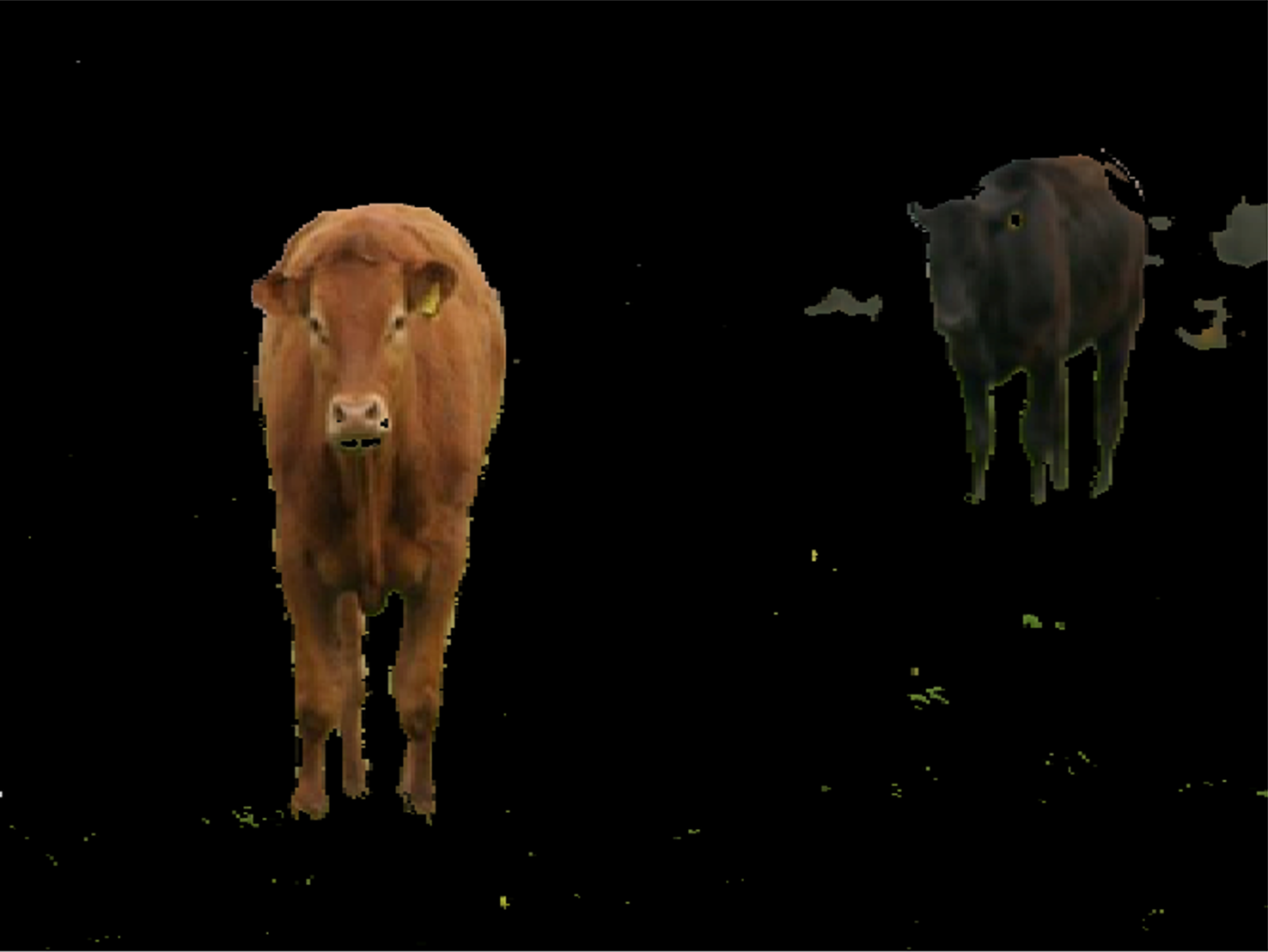}
	\caption{Image masked with MBO segmentation for \Cref{ex_twocows}, reproduced from \cite[Fig. 5.9(d)]{Buddthesis}. The segmentation accuracy is $98.4622\%$.  \label{fig:twocowsseg}}
\end{figure}
\subsection{Computational set-up}
All programming was done in \textsc{Matlab}R2020b with relevant toolboxes the Computer Vision Toolbox Version 10.1, Image Processing Toolbox Version 11.4, Signal Processing Toolbox Version 8.7, and Deep Learning Toolbox Version 14.3. In this section, all functions in \texttt{typewriter font} will refer to built-in \textsc{Matlab} functions.

Timings were taken with implementations executed serially on a machine with an Intel\textsuperscript{\textregistered}  Core\textsuperscript{\texttrademark} i7-9800X @ 3.80 GHz [16 cores] CPU and 32 GB RAM of memory.
\subsection{The feature map and its adjoint} \label{fmap}

In the following applications, we will define $\mathcal{F}$ (and likewise $\mathcal{F}_d$ \emph{mutatis mutandis}) as follows. Recall that  $x : Y \rightarrow \mathbb{R}^\ell$. For each pixel $i\in Y$, suppose we have a map $\mathcal{N}_i:\{1,...,k\} \to Y$ which defines the $k$ ``image-neighbours'' of $i$ in $Y$ and we likewise have a kernel $\mathscr{K}:\{1,...,k\} \to \mathbb{R}$. Then for each channel $s \in \{ 1,...,\ell\}$ of $x$, $i\in Y$, and $p\in \{1,...,k\}$, we define
$(\mathcal{Z}(x^s))_{ip} := \mathscr{K}(p)x^s_{\mathcal{N}_i(p)}$, and we define $z = \mathcal{F}(x)\in \mathbb{R}^{N \times k\ell}$ by $z := \begin{pmatrix} \mathcal{Z}(x^1) &\mathcal{Z}(x^2) &... &\mathcal{Z}(x^\ell)\end{pmatrix}$. We now derive $\mathcal{F}^*:\mathbb{R}^{N\times k\ell}\to \mathbb{R}^{N\times \ell}$, the adjoint of $\mathcal{F}$ with respect to $\langle\cdot,\cdot\rangle_F$. 
Write $w = \begin{pmatrix} w^1 &w^2 &... &w^\ell \end{pmatrix}$, where $w^s \in \mathbb{R}^{N\times k}$. 
Then since
\begin{equation*}
	\langle \mathcal{Z}(x^s),w^s\rangle_F = \sum_{i \in Y} \sum_{p=1}^k x^s_{\mathcal{N}_i(p)} \mathscr{K}(p)w^s_{ip}
	= \sum_{j \in Y} x^s_j \left(\sum_{\{(i,p) \mid \mathcal{N}_i(p) = j\}} \mathscr{K}(p)w^s_{ip} \right),
\end{equation*}
it follows that $\mathcal{Z}$ has adjoint $\mathcal{Z}^*:\mathbb{R}^{N\times k}\to \mathbb{R}^{N}$ given by
\[
(\mathcal{Z}^*(w^s))_j = \sum_{\{(i,p) \mid \mathcal{N}_i(p) = j\}} \mathscr{K}(p)w^s_{ip}. 
\]
Finally, by construction, $\mathcal{F}^*(w) = \begin{pmatrix} \mathcal{Z}^*(w^1) &\mathcal{Z}^*(w^2) &... &\mathcal{Z}^*(w^\ell) \end{pmatrix}$.

For this section we will take $k = 9$, the image-neighbours of pixel $i$ to be the $3\times 3$ square centred on $i$ (with replication padding at the boundaries), and $\mathscr{K}$ to be 9 multiplied by a $3\times 3$ Gaussian kernel with standard deviation 1, centred on the centre of that square, computed via \texttt{fspecial(`gaussian')}. As the images are RGB: $\ell = 3$, and $z \in \mathbb{R}^{N \times 27}$. 
\subsection{Measures of reconstruction and segmentation accuracy} 

In the following examples, we will measure the accuracy of a reconstruction $x$ relative to a ground truth of $x^*$ by the Peak Signal-to-Noise Ratio (PSNR) (defined as in \cite[(2.6)]{Hansen}). 

We will measure the accuracy of a segmentation $u$ relative to a ground truth $u^*$ by its Dice score, defined as
$2TP/(2TP + FP + FN)$, where $TP$ is the number of pixels which are true positives (``positives'' will here be pixels identified as ``cow''), $FP$ of false positives, and $FN$ of false negatives. That is (since ``positive'' pixels will be given label ``1''):
\[
\text{Dice score of } u \text{ relative to }u^* 
:= \frac{2u\cdot u^*}{2u\cdot u^* + u\cdot (1-u^*) + (1-u)\cdot u^*}.
\]
\subsection{Denoising the ``two cows'' image}
As our first application, we will test our method on a highly noised version of \Cref{ex_twocows}.
\begin{example}[Noised two cows]
	\label{ex_noisycows}
	Let $Z$, $f$, and $x^*$ (the true image that is to be reconstructed and segmented) be as in \Cref{ex_twocows}. Let the observed data $y$ (see \Cref{Fig_noisedcows}) be $x^*$ plus Gaussian noise with mean $0$ and standard deviation $1$, created via \emph{$\texttt{imnoise}$}.
	Thus, $\mathcal{T}$ is the identity. This is a very high noise level, with a typical PSNR of $y$ relative to $x^*$ of  $6.2$ dB.  
\end{example}
\begin{figure}[h]
	\centering
	\includegraphics[width = 0.6\textwidth]{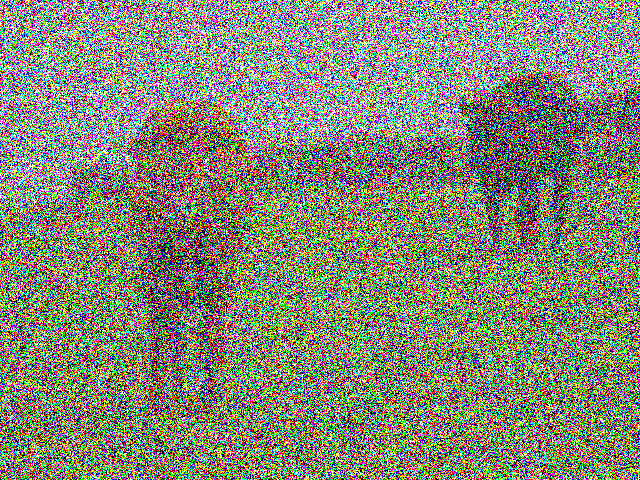}
	\caption{Typical observed data $y$ for \Cref{ex_noisycows}.} \label{Fig_noisedcows}
\end{figure}

\subsubsection{Parameters and initialisation}\label{twocowsparams}

Unless otherwise stated, all parameter values were chosen through manual trial-and-error. 
We let $\alpha = 0.75$, $\beta = 10^{-5}$, $\eta_n = 0.1$, $\nu_n = 10^{-6}$, and $\sigma = 3$. For the SDIE scheme for \cref{uupdate2}, we choose $\tau = \varepsilon = 0.00285$, $\mu = 50\chi_Z$, $k_s = 5$, and stopping condition parameter $\delta = 10^{-10}$. For the Nystr\"om-QR scheme we take $K = 100$. 
As regulariser $\mathcal{R}$ we use Huber-TV \cite{Huber}:\footnote{This choice of regulariser from among the various commonly used candidates is somewhat arbitrary; we have not yet explored the impact of the choice of regulariser on the behaviour of this scheme.} 
\[
\mathcal{R}(x) = R(\mathcal{K}x) = 10 \sum_{i\in Y} \begin{cases}
	\|(\nabla x)_i\|_2 - 0.005, &\text{if }\|(\nabla x)_i\|_2 > 0.01,\\
	\|(\nabla x)_i\|^2_2/0.02, &\text{if }\|(\nabla x)_i\|_2 \leq 0.01,
\end{cases}
\] 
where $\mathcal{K}x:= \nabla x$, with $(\nabla x)_i = (
x_{i\downarrow}-x_i,
x_{i\rightarrow}-x_i 
)^T$, where $x_{i\downarrow}$ is the $x$-value at the pixel directly below pixel $i$ in the image and $x_{i\rightarrow}$ the $x$-value at the pixel directly to the right of pixel $i$ (with replication padding at boundaries).

The initial reconstruction $x_0$ was computed via a standard TV-based (i.e., Rudin--Osher--Fatemi \cite{ROF}) denoising, with fidelity term 1.05. That is 
\[
x_0 = \argmin_{x \in \mathbb{R}^{N\times \ell}} \operatorname{TV}(x) + 1.05\|x - y\|_F^2,
\]
where $\TV(x):=\sum_{i\in Y} \|(\nabla x)_i\|_2$
; this is solved using the split Bregman method from Getreuer \cite{getreuer2012tvdenoising}, using code from \url{https://getreuer.info/posts/tvreg/index.html} (accessed 10 August 2022), with 50 split Bregman iterations and a tolerance of $10^{-5}$. 
The initial segmentation $u_0$ of $x_0$ was computed via the SDIE scheme with the above parameters and initial state $u^0 = 0.47\chi_Y + f$. 

\subsubsection{Example results} \label{subsubsec_example_denoise_segmen}

Before discussing the choice of parameters, timings, and accuracy in more detail, we present an example run of the reconstruction-segmentation method for the noisy data and set-up described in \Cref{ex_noisycows} and \Cref{twocowsparams}, respectively. 

\begin{figure}
	\centering 
	\includegraphics[width = \textwidth]{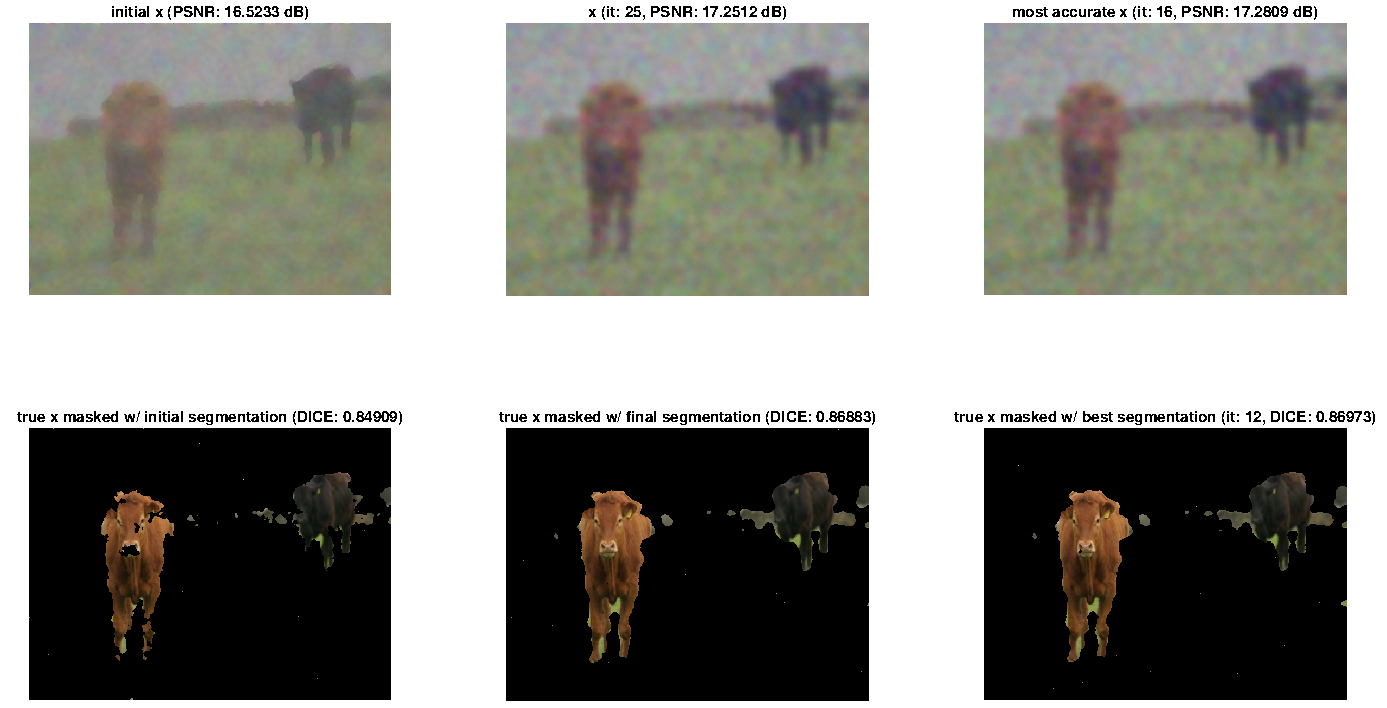}
	\caption{Reconstruction and segmentation of \Cref{ex_noisycows} from a single run of the reconstruction-segmentation algorithm using the parameters from \Cref{twocowsparams}. Each column contains the reconstruction (top) and segmentation (bottom). We show (left to right): initialisation, final iteration (=25), and the output with the best scores over the run. } \label{fig:RSdenoising} 
\end{figure}

We show the results of our denoising-segmentation in \Cref{fig:RSdenoising}. In the figure we see that the reconstruction is not particularly good by itself (as one would expect given the very high noise level), but the segmentation is very good (compared to the baseline achieved by other methods, see \Cref{sec:seqdenoise}). Importantly, the reconstruction increases the contrast between cows and background, which aids the segmentation.

\subsubsection{Parameters, accuracy, and timings} \label{sec:denoresu}

We now study the denoising-segmentation of \Cref{ex_noisycows} more quantitatively. We consider timings of the total runs, but also of the most important steps, as well as reconstruction and segmentation accuracy. In order to understand the dependence on the algorithm's parameters we consider four settings: the setting from \Cref{twocowsparams}, a change in the segmentation parameters compared to \Cref{twocowsparams} ($K = 70$ (decrease), $\varepsilon =  \tau = 0.003$ (increase)) (as used in \cite{Budd3}), a doubling of the segmentation weight ($\beta = 2\cdot 10^{-5}$), and a decrease in the reconstruction weight ($\alpha =0.5 $). We list the results of these settings in \Cref{tbl:denoiseresults}, and present them in \Cref{fig:ex_noisy_recons,fig:results1,fig:results2,fig:results3,fig:results4}.

\begin{table}[]
\begin{tabular}{ll|rrrr}
\multirow{2}{*}{}         & Changes to \ref{twocowsparams}      & $\emptyset$ & $K, \varepsilon, \tau$  & $\beta$ & $\alpha$  \\ \hline
                          & Figure              & \ref{fig:results1} &\ref{fig:results2}  & \ref{fig:results3} & \ref{fig:results4} \\ \hline
\multirow{2}{*}{Accuracy} & Dice score [\%]    &  $85.84 (1.68)$ & $85.46 (1.28)$ & $87.14 (1.45)$  & $86.51 (1.59)$ \\
                          & PSNR [dB]    &$17.19 (0.03)$  & $17.22 (0.03)$ & $17.74 (0.05)$ & $17.54 (0.04)$ \\ \hline
\multirow{4}{*}{Timings [s]}  & total               &$150.19 (1.52)$  &$128.86 (1.23)$ & $154.00(1.81)$  & $153.80 (1.25)$ \\
                          & initialisation      & $1.49 (0.03)$  & $1.52 (0.03)$  &  $1.51 (0.08)$ &  $1.49 (0.03)$  \\
                          & reconstruction   & $3.43 (0.08)$ & $3.31 (0.07)$   & $3.57 (0.08)$ & $3.56 (0.07)$ \\
                          & segmentation & $2.52 (0.03)$ & $1.79 (0.03)$ & $2.53 (0.04)$ & $2.53 (0.03)$
\end{tabular}
\caption{Results averaged over 50 runs for different parameter settings for \Cref{ex_noisycows}. Values are given as ``mean(standard deviation)''. Total timing is for 25 iterations of the algorithm. Details of the changes are described in the main text.
\label{tbl:denoiseresults}}
\end{table}

\begin{figure}
    \centering
    \includegraphics[width = \textwidth,  clip]{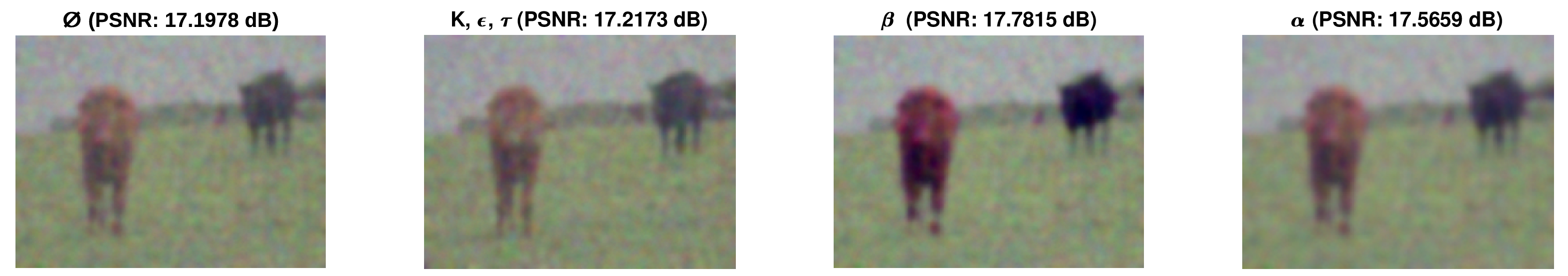} 
    \caption{Example reconstruction results obtained from the denoising-reconstruction method for the different parameter settings for \Cref{ex_noisycows} described in \Cref{sec:denoresu}.}
    \label{fig:ex_noisy_recons}
\end{figure}

In all of the settings, we obtain a quick and accurate segmentation. The PSNR of the reconstruction is similar in all settings, 
suggesting that this low accuracy is intrinsic to this noise level. The initial segmentation is very fast, but in all settings not as accurate as the joint reconstruction-segmentation. Of the results, some observations are particularly remarkable:
\begin{itemize}
    \item For the pure MBO segmentation discussed in \cite{Budd3}, the  Nystr\"om rank $K$ and $\tau = \varepsilon$ were chosen as in the second setting presented here (\Cref{fig:results2}). Whilst this choice was optimal for pure segmentation, and the smaller rank leads to a smaller computational cost, the obtained Dice scores are worse than all the other settings.
\item The larger $\beta$ value (\Cref{fig:results3}, see also \Cref{fig:ex_noisy_recons} third from left) puts more weight on the Ginzburg--Landau energy, leading to the best Dice average, but (unexpectedly) also the highest PSNR average. 
\item The Dice values 
have relatively large standard deviations. This is likely caused by the very high noise level, but also by the inherent randomness of the Nystr\"om extension.
\item We attain near peak Dice and PSNR values in fewer than ten iterations in all cases. 
\end{itemize}

\begin{figure} 
    \centering
    \includegraphics[width = \textwidth, trim = 500 0 500 0, clip]{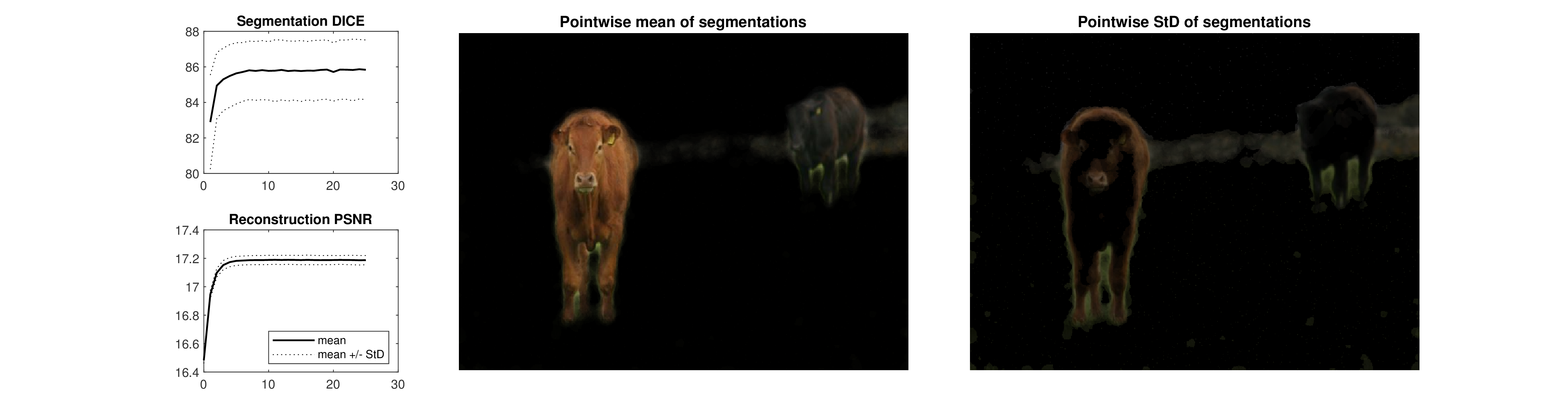}
    \caption{Results averaged over 50 runs of the reconstruction-segmentation scheme in the denoising case. The line plots (left) show the evolution of the Dice score (\%) and PSNR (dB) over the 25 iterations of the algorithm. The images (centre and right) show mean and standard deviation of the segmentation at $t=25$ which are used as weights for the ground truth images. The setting is that of \Cref{twocowsparams}. 
    }
    \label{fig:results1}
\end{figure}

\begin{figure} 
    \centering
    \includegraphics[width = \textwidth, trim = 500 0 500 0, clip]{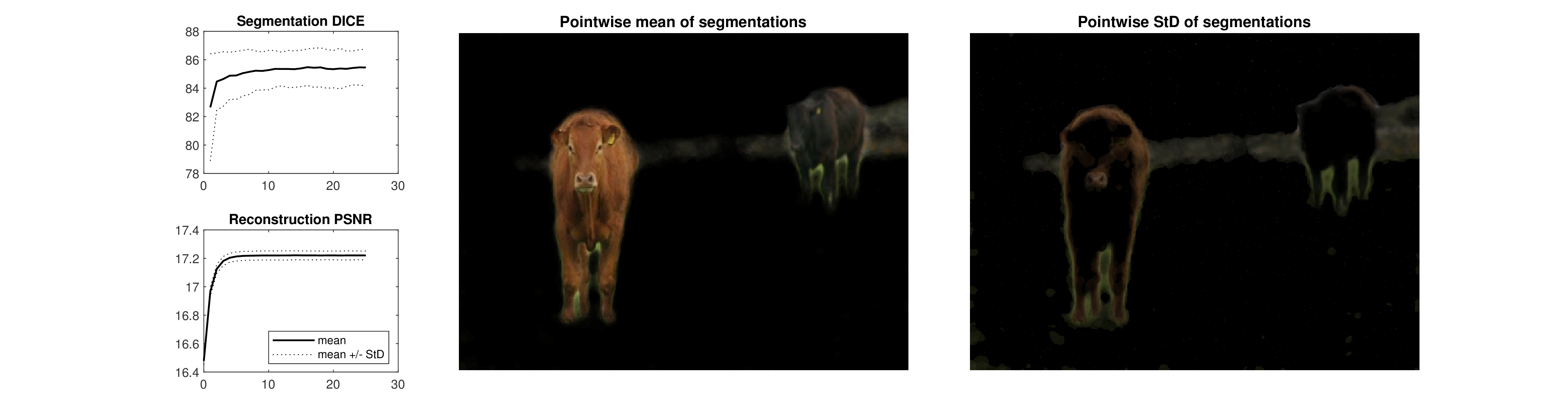}
    \caption{Results averaged over 50 runs of the reconstruction-segmentation scheme in the denoising case. The line plots (left) show the evolution of the Dice score (\%) and PSNR (dB) over the 25 iterations of the algorithm. The images (centre and right) show mean and standard deviation of the segmentation at $t=25$ which are used as weights for the ground truth images. The setting is that of \Cref{twocowsparams}, subject to the changes: $K = 70$, $\tau = \varepsilon = 0.003$. 
    }
    \label{fig:results2}
\end{figure}

\begin{figure} 
    \centering
    \includegraphics[width = \textwidth, trim = 500 0 500 0, clip]{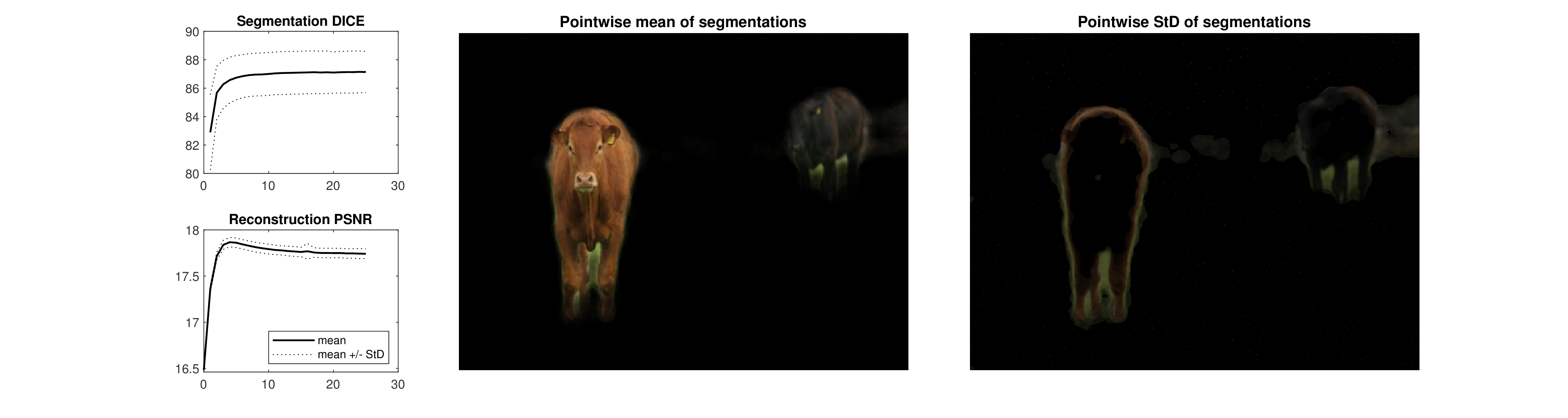}
    \caption{Results averaged over 50 runs of the reconstruction-segmentation scheme in the denoising case. The line plots (left) show the evolution of the Dice score (\%) and PSNR (dB) over the 25 iterations of the algorithm. The images (centre and right) show mean and standard deviation of the segmentation at $t=25$ which are used as weights for the ground truth images. The setting is that of \Cref{twocowsparams}, subject to the change: $\beta = 2 \cdot 10^{-5}$. 
    }
    \label{fig:results3}
\end{figure}

\begin{figure} 
    \centering
    \includegraphics[width = \textwidth, trim = 500 0 500 0, clip]{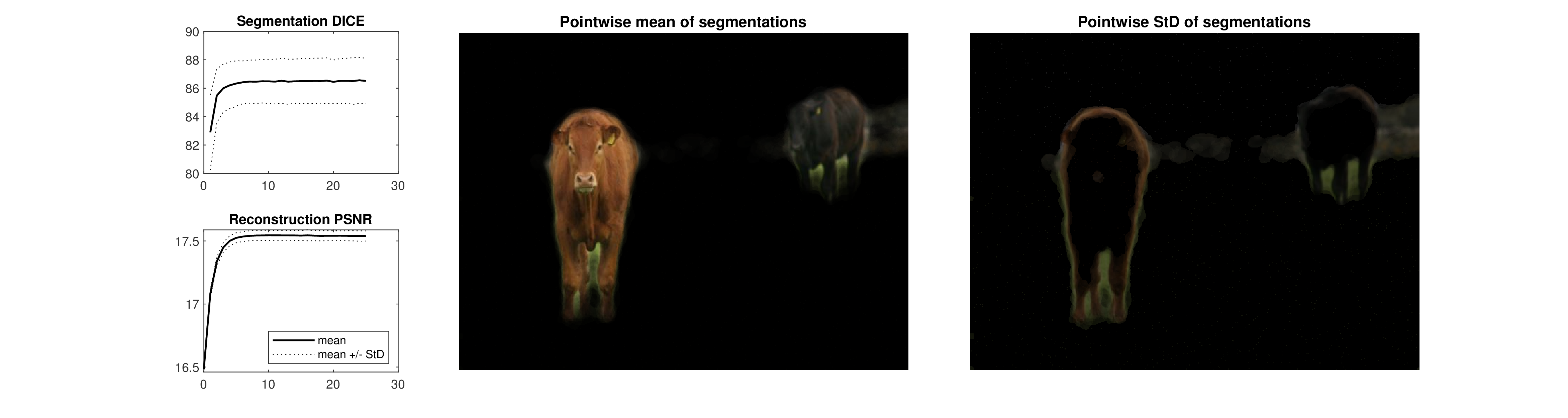}
    \caption{Results averaged over 50 runs of the reconstruction-segmentation scheme in the denoising case. The line plots (left) show the evolution of the Dice score (\%) and PSNR (dB) over the 25 iterations of the algorithm. The images (centre and right) show mean and standard deviation of the segmentation at $t=25$ which are used as weights for the ground truth images. The setting is that of \Cref{twocowsparams}, subject to the change:  $\alpha = 0.5$. 
    }
    \label{fig:results4}
\end{figure}

\subsubsection{Comparison to sequential reconstruction-segmentation}\label{sec:seqdenoise}
Finally, we compare the accuracy of our joint reconstruction-segmentation approach to the more traditional sequential approach. That is, we will first denoise the image of cows, and then segment it. Segmentations will be performed using the graph MBO scheme with the same set-up as in \Cref{twocowsparams}.

One example of this sequential approach is our initialisation process (i.e., TV denoising followed by MBO segmentation), which we observed gave worse PSNR and Dice scores than the joint reconstruction-segmentation output. However, this is perhaps unfair because that initialisation was specifically chosen because it is quick (around $1.5s$, see \Cref{tbl:denoiseresults}), whilst the whole joint reconstruction-segmentation scheme takes about 2.5 minutes to run (although as was mentioned above, the scheme achieves near peak accuracy in closer to 1 minute).

For a potentially fairer comparison, we consider three more sophisticated denoisers: denoising with total generalised variation (TGV) regularisation \cite{BKPock2011} through code by Condat \cite{Condat2013} and downloaded from \url{https://lcondat.github.io/software.html} (accessed 21 May 2022); the block-matching and 3D filtering (BM3D) algorithm \cite{MakinenAzzariFoi2020} via code by Mäkinen, Azzari, and Foi and downloaded from \url{https://webpages.tuni.fi/foi/GCF-BM3D/} (accessed 18 May 2022); and the built-in \textsc{Matlab} pre-trained denoising Convolutional Neural Network (CNN) (defined in Zhang \emph{et al.} \cite{DnCNN}) loaded via \texttt{denoisingNetwork(`DnCNN')} and implemented via \texttt{denoiseImage}. Remarkably, however, we observe that the BM3D denoiser barely outperforms TV, the TGV denoiser performs slightly worse than TV, and the CNN denoiser performs much worse,\footnote{We expect this poor performance to be because the network is pre-trained; this performance is somewhat concerning, and may
not be representative of the general capabilities of CNNs for this denoising task.} see \Cref{fig:denoisers}. Comparing with the PSNR scores in \Cref{tbl:denoiseresults}, we note that the TV, TGV, BM3D, and CNN denoisers are all outperformed by the reconstruction from our scheme.
\begin{figure}[ht]
	\centering
	\includegraphics[width = \textwidth]{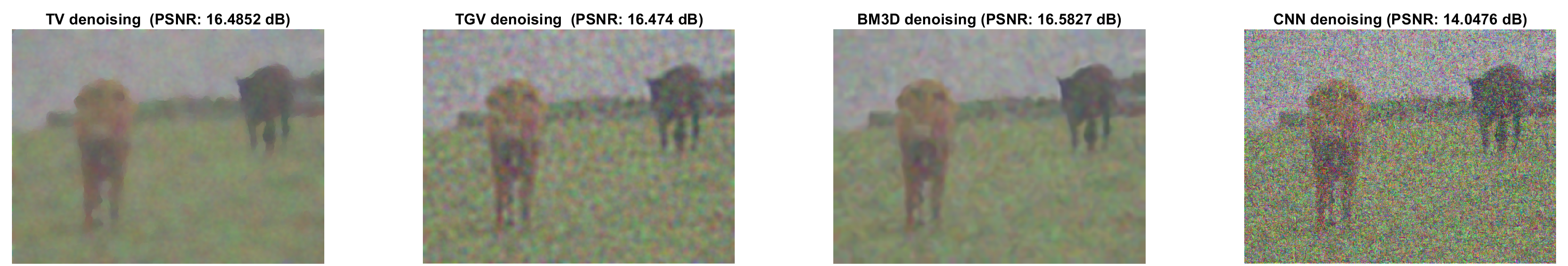}
	\caption{Typical denoised output for the TV-based, TGV-based, BM3D, and CNN denoisers.} \label{fig:denoisers}
\end{figure}

The mean Dice score $\pm$ standard deviation (over 50 trials) for MBO segmentations of the TGV denoised image is $0.7138 \pm 0.0578$, of the BM3D denoised image is $0.8180 \pm 0.0130$, and of the CNN denoised image is $0.5167\pm 0.0094$. 
Typical segmentations are shown in \Cref{fig:BM3D_CNN_segs}. Similarly to the reconstructions, all of the segmentations are worse or considerably worse than the segmentations obtained from our scheme. However, the sequential denoising-segmentations are notably faster than our scheme; we report the timings in \Cref{tbl:seqtimings}.  
\begin{figure}[h]
	\centering 
	\includegraphics[width = \textwidth]{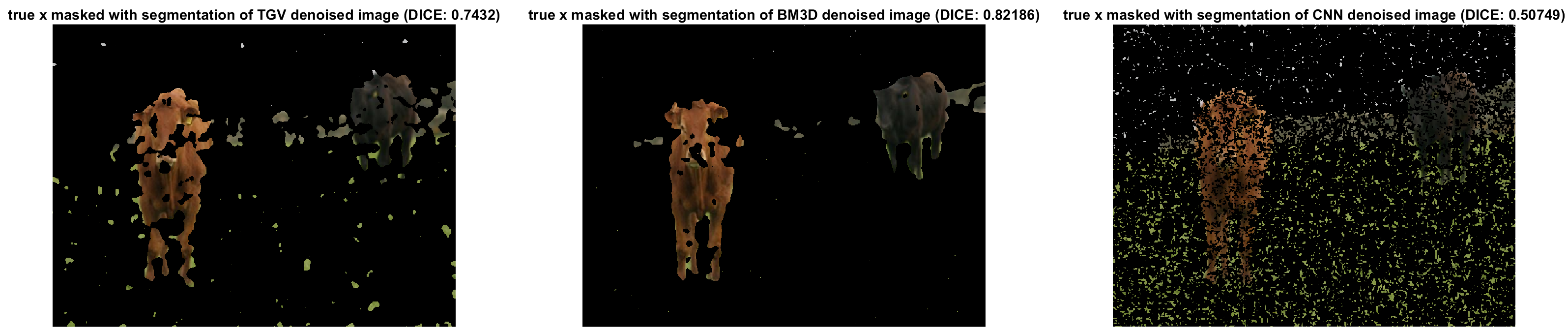}
	\caption{The ground truth image masked with typical MBO segmentations of the TGV, BM3D, and CNN denoised image.} \label{fig:BM3D_CNN_segs} 
\end{figure}
\begin{table}[h] \centering
\begin{tabular}{|l|r|r|r|}
\hline
                            & TGV         & BM3D        & CNN            \\ \hline
Reconstruction time {[}s{]} & 8.51(0.05)  & 14.67(0.14)  & 0.19(0.14)  \\ \hline
Segmentation time {[}s{]}   & 5.99(0.89)  & 5.84(0.04) & 5.87(0.03)  \\ \hline
Total time {[}s{]}          & 14.50(0.89) & 20.51(0.14) & 6.06(0.13)  \\ \hline
\end{tabular}
\caption{Timings averaged over 50 runs for the three sequential denoising-segmentation methods. 
Values are given as ``mean(standard deviation)''. 
\label{tbl:seqtimings}}
\end{table}

\subsection{Deblurring the ``two cows''}
For our next example, now with $\mathcal{T}$ not equal to the identity map, we consider a blurred version of \Cref{ex_twocows}.
\begin{example}[Blurred two cows]
	\label{ex_blurrycows}
	Let $Z$, $f$, and the true image $x^*$ (that is to be reconstructed and segmented) be as in \Cref{ex_twocows}. Let the observed data $y$ (see \Cref{Fig_blurredcows}) be a horizontal motion blurring of $x^*$ of distance $75$ pixels (with symmetric padding at the boundary) created via \emph{$\texttt{imfilter}$}, plus Gaussian noise with mean $0$ and standard deviation $10^{-1}$ created via \emph{$\texttt{imnoise}$}. This $y$ has a typical PSNR relative to $x^*$ of around $17.9$ dB.
\end{example}
\begin{figure}[h]
	\centering
	\includegraphics[width = 0.6\textwidth]{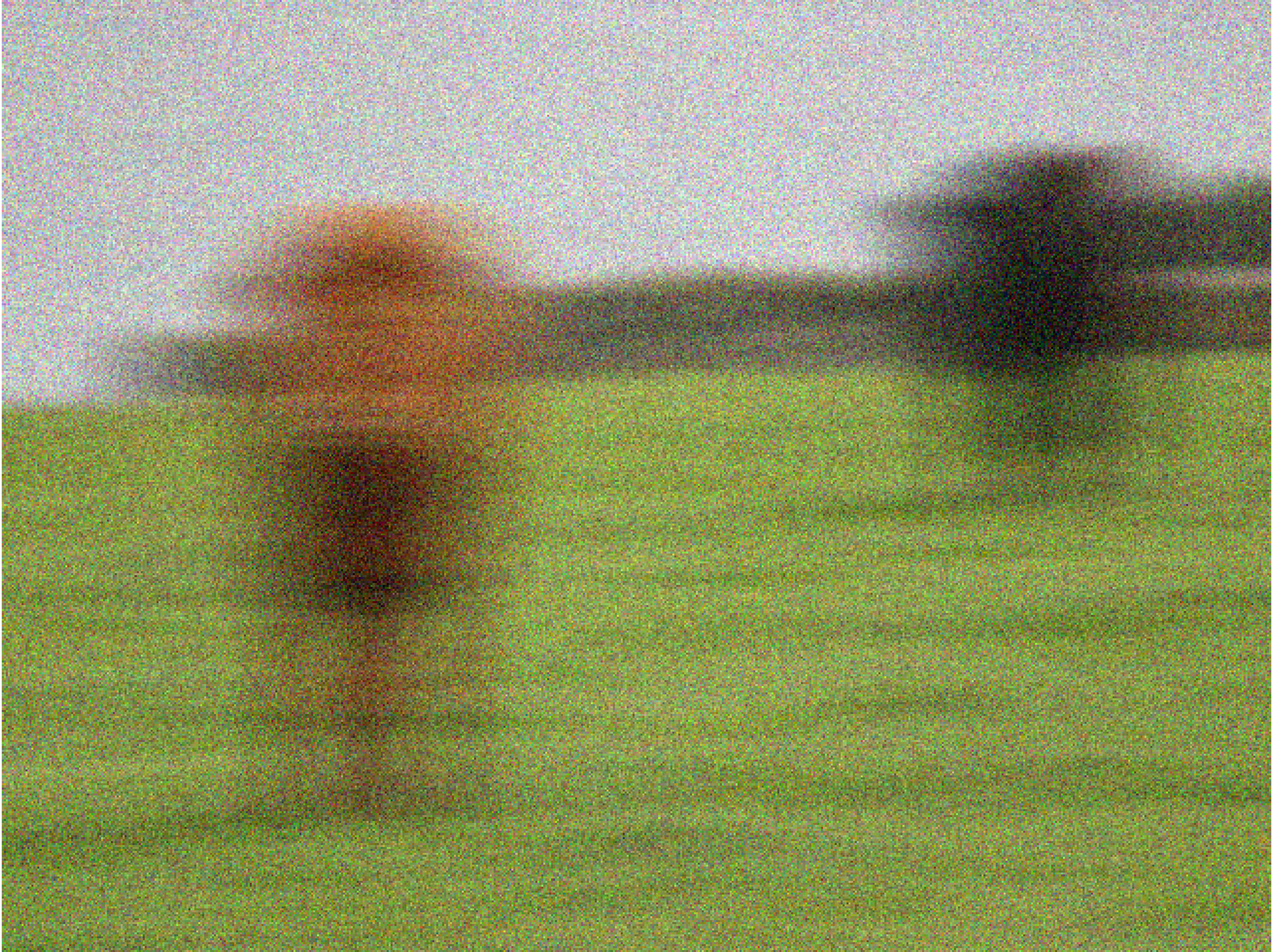}
	\caption{Typical observed data $y$ for \Cref{ex_blurrycows}.} \label{Fig_blurredcows}
\end{figure}
\subsubsection{$\mathcal{T}$, its adjoint, and \cref{proxG}}

In \Cref{ex_blurrycows}, the forward model $\mathcal{T}$ works by convolving $x$ with a motion blur filter $\mathcal{M}$ (computed using \texttt{fspecial(`motion')}). The adjoint $\mathcal{T}^*$ therefore corresponds to convolution with a filter $\mathcal{M}'$ defined by reflecting $\mathcal{M}$ in both axes. In the case of motion blur, $\mathcal{M} = \mathcal{M}'$, so $\mathcal{T}$ is self-adjoint. Furthermore, $\mathcal{M}$ has non-negative values, and so $\mathcal{T}$ is represented by a non-negative (symmetric) matrix.  

In order to solve \cref{xupdate3}, recall that we need to compute \cref{proxG}. That is, we need to be able to compute solutions to equations of the form $
\left( \left(2\eta_n + \delta t^{-1}\right) I + 2\alpha \mathcal{T}^2\right)x = z.
$   
We will do this via a fixed-point iteration. That is, we let $x^0 := z$ and iterate 
\begin{equation}\label{fpi}
	x^{m+1} := \frac{1}{2\eta_n + \delta t^{-1}} z - \frac{2\alpha}{2\eta_n + \delta t^{-1}}\mathcal{T}^2x^m. 
\end{equation}
By the Banach fixed-point theorem, if $\zeta:=2\alpha(2\eta_n + \delta t^{-1})^{-1}\|\mathcal{T}^2\|<1$, 
then $\|x - x^m\|_F = \bigO(\zeta^m)$ as $m \to \infty$. As $\|\mathcal{T}\| = 1$,\footnote{ As $\mathcal{T}$ is self-adjoint, $\|\mathcal{T}\|$ is the modulus of the largest eigenvalue of $\mathcal{T}$. Let $\mathcal{T}x = \lambda x$, where $\lambda \in \mathbb{R}$ since $\mathcal{T}$ is real and symmetric. Since $\mathcal{T}(\mathbf{1}) = \mathbf{1}$, $\mathcal{T}$ is represented by a non-negative matrix, and $-\|x\|_{\infty}\mathbf{1} \leq x \leq \|x\|_{\infty}\mathbf{1}$ elementwise, by applying $\mathcal{T}$ to the latter we get 
$
-\|x\|_{\infty}\mathbf{1} \leq \lambda x \leq \|x\|_{\infty}\mathbf{1}
$
elementwise, and so $|\lambda|\leq 1$.} it suffices to take $\alpha < \eta_n + \frac12 \delta t^{-1}$ for convergence of \cref{fpi}. 
\subsubsection{Parameters and initialisation} \label{sec:blurparams}
We take $\alpha = \eta_n = 2$, $\tau = \varepsilon = 0.002$, $K = 200$, and parameters $\beta$, $\nu_n$, $\sigma$, $\mu$, $k_s$, and $\delta$ as in \Cref{twocowsparams}.
We take $\mathcal{R}$ as in \Cref{twocowsparams} except that we change the multiplicative factor from $10$ to $1$. Another change is the number of iterations of the algorithm: while we iterate for 25 steps in the denoising problems, in preliminary runs (not reported) we noticed in the deblurring case that 15 steps are sufficient.

The initial reconstruction $x_0$ is computed via TV deblurring with fidelity 
term 45, 
i.e.,
\[
x_0 = \argmin_{x \in \mathbb{R}^{N\times \ell}} \operatorname{TV}(x) + 45\|\mathcal{T}(x) - y\|_F^2.
\]
This is solved by the split Bregman method of Getreuer \cite{getreuer2012tvdeconvolution}, using code from \url{https://getreuer.info/posts/tvreg/index.html} (accessed 10 August 2022), with 50 split Bregman iterations and a tolerance of $10^{-5}$. 
The initial segmentation $u_0$ of $x_0$ is computed via the SDIE scheme with the above parameters and initial state $u^0 = 0.47\chi_Y + f$. 

\subsubsection{Example results} \label{subsubsec_example_deblur_segmen}

Before discussing the choice of parameters, timings, and accuracy in more detail, we present an example run of the reconstruction-segmentation method for the noisy data and set-up discussed in the beginning of this section. Here, we use the parameter setting from \Cref{sec:blurparams}. We show the results of this example run in \Cref{fig:RSdeblur}.

\begin{figure}
	\centering 
	\includegraphics[width = \textwidth]{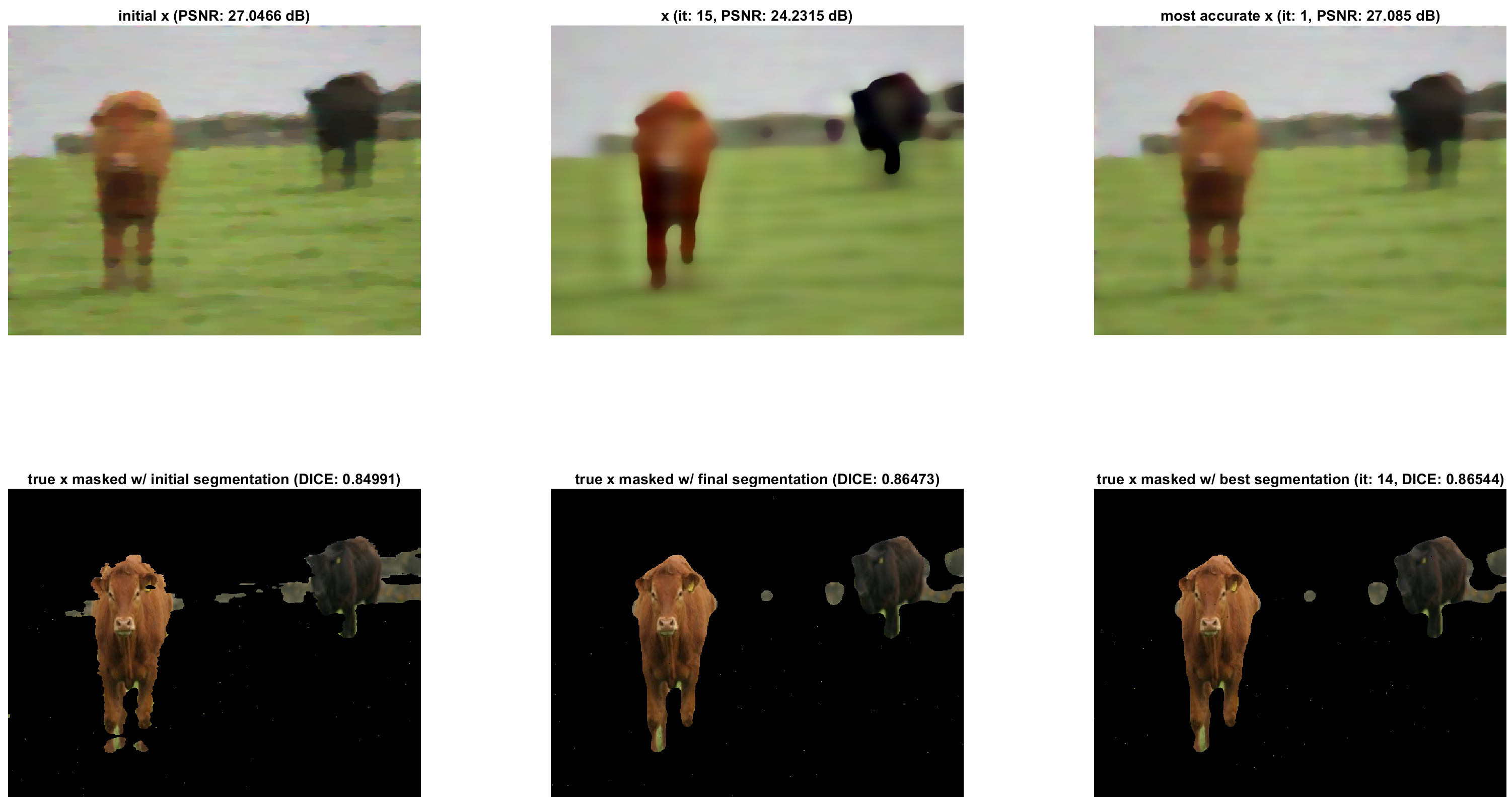}
	\caption{Reconstruction and segmentation of \Cref{ex_blurrycows} from a single run of the reconstruction-segmentation algorithm using the parameters from \Cref{sec:blurparams}. Each column contains the reconstruction (top) and segmentation (bottom). We show (left to right): initialisation, final iteration (=15), and the output with the best scores over the run. 
	} \label{fig:RSdeblur} 
\end{figure}

The effect of an increasing contrast between cows and background in the reconstruction, which we already observed for the denoising results in \Cref{subsubsec_example_denoise_segmen}, is even more visible in this deblurring reconstruction. In contrast to the denoising setting, here this even leads to reconstructions with PSNRs that deteriorate over the runtime of the algorithm. 

\subsubsection{Parameters, accuracy, and timings} \label{sec:deblresu}
As in \Cref{sec:denoresu}, we now study the deblurring-segmentation of \Cref{ex_blurrycows} more quantitatively. We again consider timings of the total runs and of the key steps, as well as reconstruction and segmentation accuracy. Again, we look at four parameter settings: the setting from \Cref{sec:blurparams}, a change in the segmentation parameters compared to \Cref{sec:blurparams} ($K = 100$ (decrease), $\varepsilon =  \tau = 0.00285$ (increase), $u^0=0.45\chi_Y + f$ (lower constant on $Y$)), an increase in the segmentation weight ($\beta = 1.52\cdot 10^{-5}$), and an increase in the reconstruction parameters ($\alpha =\eta_n = 10 $). We list the results of these settings in \Cref{tbl:deblurresults}, and present them in \Cref{fig:ex_blurry_recons,fig:results1blurr,fig:results2blurr,fig:results4blurr,fig:results3blurr}.
\begin{table}[]
\begin{tabular}{ll|rrrr}
\multirow{2}{*}{}         & Changes to \ref{sec:blurparams}      & $\emptyset$ & $K, \varepsilon, \tau, u^0$  & $\beta$ & $\alpha,  \eta_n$  \\ \hline
                          & Figure              & \ref{fig:results1blurr} &\ref{fig:results2blurr}  & \ref{fig:results4blurr} & \ref{fig:results3blurr} \\ \hline
\multirow{2}{*}{Accuracy} & Dice score [\%]    &  $86.99 (0.63)$ & $77.23 (12.96)$ & $85.11 (0.98)$  & $85.79 (0.82)$ \\
                          & PSNR [dB]    &$24.16 (0.15)$  & $22.30 (3.52)$ & $18.80 (0.23)$ & $27.32 (0.06)$ \\ \hline
\multirow{4}{*}{Timings [s]}  & total               &$179.73 (0.97)$  &$124.21 (4.73)$ & $179.60 (0.58)$  & $177.53 (0.39)$ \\
                          & initialisation      & $3.31 (0.09)$  & $3.93 (0.18)$  &  $3.30 (0.09)$ &  $3.29 (0.06)$  \\
                          & reconstruction   & $5.91 (0.11)$ & $5.44 (0.77)$   & $5.97 (0.09)$ & $5.87 (0.06)$ \\
                          & segmentation & $5.85 (0.08)$ & $2.58 (0.24)$ & $5.78 (0.06)$ & $5.75 (0.06)$
\end{tabular}
\caption{Results averaged over 50 runs for different parameter settings for \Cref{ex_blurrycows}. Values are given as ``mean(standard deviation)''. Total timing is for 15 iterations of the algorithm.
\label{tbl:deblurresults}}
\end{table}

\begin{figure}
    \centering
    \includegraphics[width = \textwidth, trim = 0 0 0 0, clip]{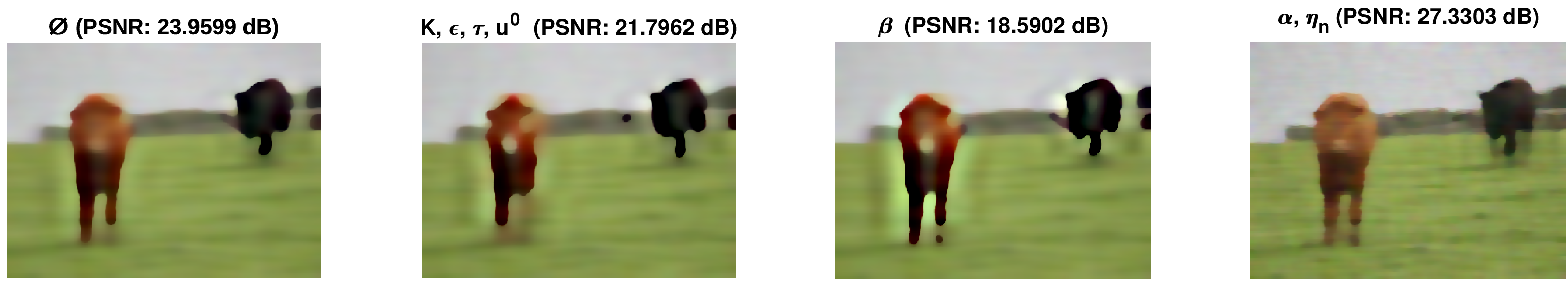}
    \caption{Example reconstruction results obtained from the deblurring-reconstruction method for the different parameter settings for \Cref{ex_blurrycows} described in \Cref{sec:deblresu}.}
    \label{fig:ex_blurry_recons}
\end{figure}

\begin{figure} 
    \centering
    \includegraphics[width = \textwidth, trim = 500 0 500 0, clip]{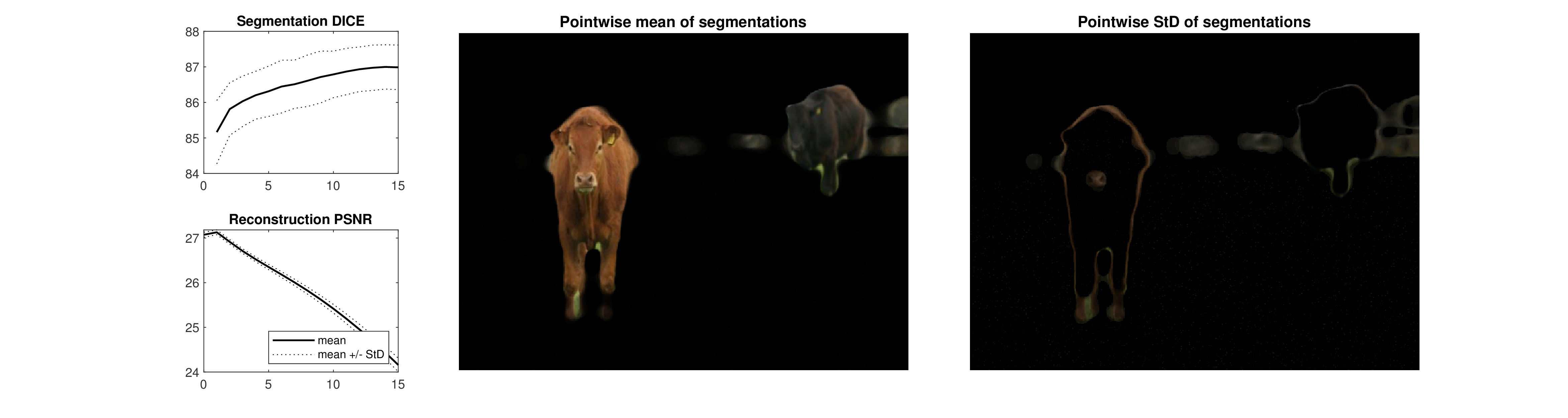}
    \caption{Results averaged over 50 runs of the reconstruction-segmentation scheme in the deblurring case. The line plots (left) show the evolution of the Dice score (\%) and PSNR (dB) over the 15 iterations of the algorithm. The images (centre and right) show mean and standard deviation of the segmentation at $t=15$ which are used as weights for the ground truth images. The setting is that of \Cref{sec:blurparams}.
}
    \label{fig:results1blurr}
\end{figure}

\begin{figure} 
    \centering
    \includegraphics[width = \textwidth, trim = 500 0 500 0, clip]{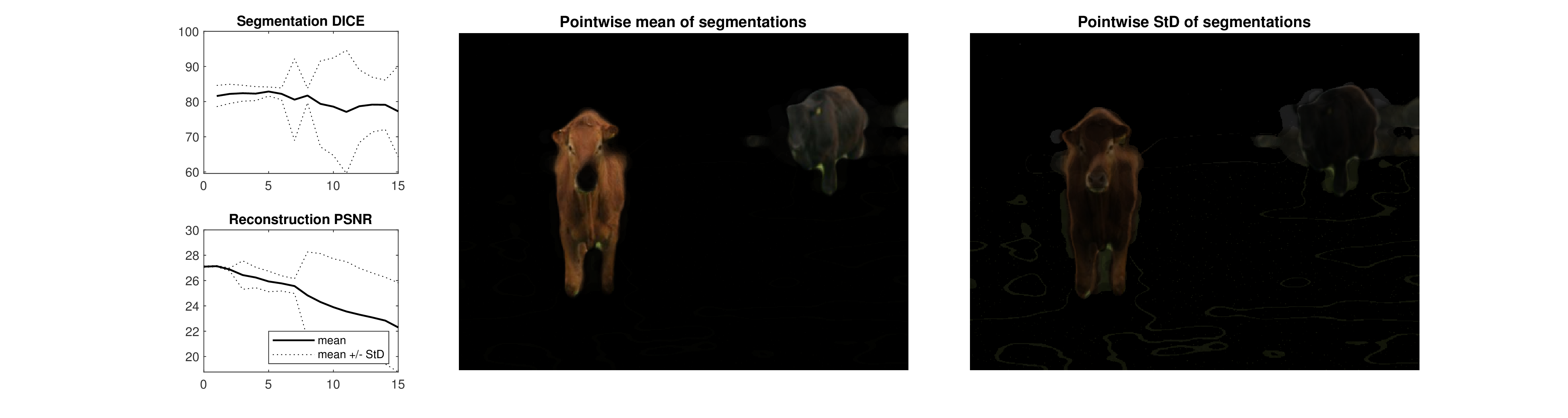}
    \caption{Results averaged over 50 runs of the reconstruction-segmentation scheme in the deblurring case. The line plots (left) show the evolution of the Dice score (\%) and PSNR (dB) over the 15 algorithm iterations. The images (centre and right) show mean and standard deviation of the segmentation at $t=15$ which are used as weights for the ground truth images.  Setting is as in \Cref{sec:blurparams}, with changes:
    $\varepsilon = \tau = 0.00285$, $K=100$, $u^0=0.45\chi_Y +f$.
    }
    \label{fig:results2blurr}
\end{figure}

\begin{figure} 
    \centering
    \includegraphics[width = \textwidth, trim = 500 0 500 0, clip]{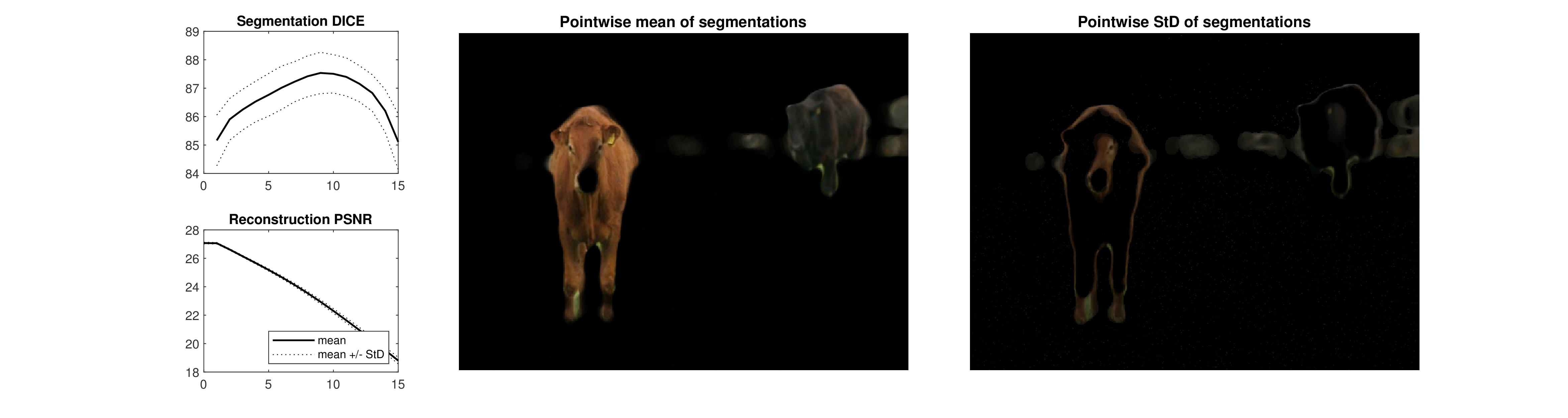}
    \caption{Results averaged over 50 runs of the reconstruction-segmentation scheme in the deblurring case.  The line plots (left) show the evolution of the Dice score (\%) and PSNR (dB) over the 15 iterations of the algorithm. The images (centre and right) show mean and standard deviation of the segmentation at $t=15$ which are used as weights for the ground truth images. The setting is that of \Cref{sec:blurparams}, subject to the change: $\beta = 1.5 \cdot 10^{-5}$.
}
    \label{fig:results4blurr}
\end{figure}

\begin{figure} 
    \centering
    \includegraphics[width = \textwidth, trim = 500 0 500 0, clip]{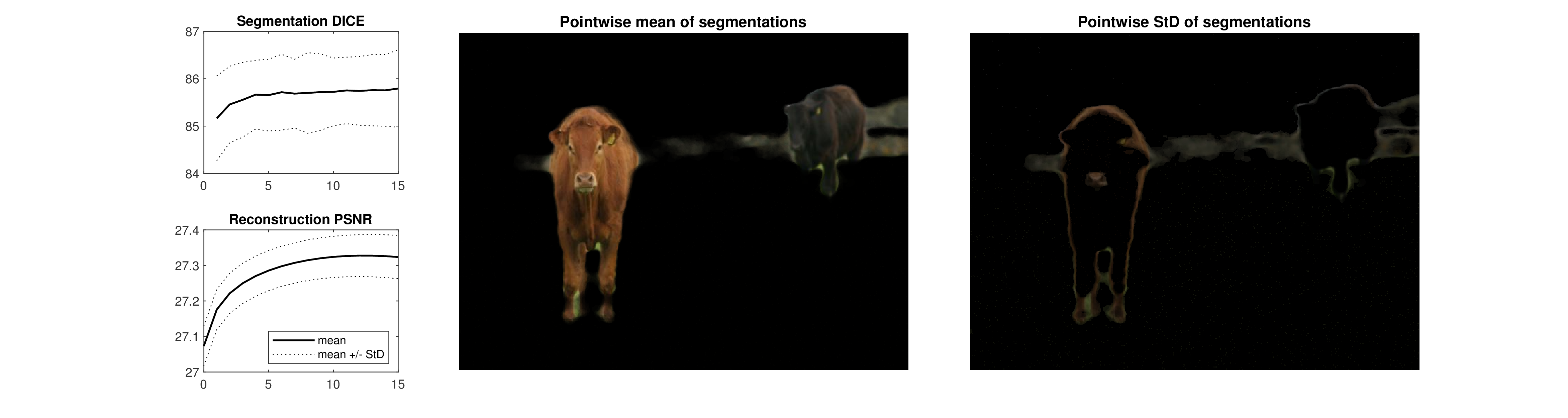}
    \caption{Results averaged over 50 runs of the reconstruction-segmentation scheme in the deblurring case. The line plots (left) show the evolution of the Dice score (\%) and PSNR (dB) over the 15 iterations of the algorithm. The images (centre and right) show mean and standard deviation of the segmentation at $t=15$ which are used as weights for the ground truth images. The setting is that of \Cref{sec:blurparams}, subject to the change:  $\alpha  = \eta_n = 10$.
}
    \label{fig:results3blurr}
\end{figure}

We now comment on those simulation results. As mentioned before, we see that in most settings, the reconstruction PSNR is reduced over the course of the algorithm. The joint reconstruction-segmentation algorithm enhances the constrast between segments to a point where the reconstruction quality suffers (see \Cref{fig:ex_blurry_recons}). Unlike in most settings we used in the denoising problem, here the segmentation accuracy is not always monotonically increasing. 

It is interesting to note that (as shown in \Cref{fig:results2blurr}) $\varepsilon, \tau$ should be chosen smaller for the deblurring problem than for \Cref{ex_noisycows} or the noise-free problem (see \cite{Budd3}). In the Allen--Cahn equation, a smaller $\varepsilon$ leads to a smaller interface and, thus, a harder thresholding. A harder thresholding should lead to a stronger regularisation which aids the deblurring. Indeed, the softer thresholding in  \Cref{fig:results2blurr} leads to more parts of the background being incorrectly identified as cow. Moreover, an even larger Nystr\"om rank $K$ is required. If $K = 100$, the results have a large variance and barely improve the initial segmentation on average. 

When increasing $\beta$, we see a slight increase in the Dice standard deviation, which might imply that the method becomes more unstable when increasing the influence of the Ginzburg--Landau energy. When increasing $\alpha$, as expected, we see an increased reconstruction accuracy. 

In the case where $\beta = 1.5\cdot 10^{-5}$ (\Cref{fig:results4blurr}), we see a certain long-term instability: the Dice score reaches its maximum at iteration step number 9, but is considerably lower at the end of the algorithm. In iteration step 9, we obtain an average Dice of $0.8754$ ($\pm 0.0073$), beating all of the results reported in \Cref{tbl:deblurresults}. Hence, the number of iterations is also an important tuning parameter as the system can experience metastability.

\subsubsection{Comparison to sequential reconstruction-segmentation}
As in the denoising case, we observe that (except in the setting of \Cref{fig:results2blurr}) our joint scheme outperforms the sequential TV-based initialisation in terms of Dice score (and in the setting of \Cref{fig:results3blurr}, also in PSNR). 

For a fairer comparison, we seek to compare the performance of our scheme to that of a sequential method using a more sophisticated deblurrer. We consider three alternative deblurrers:\footnote{We also tested the built-in \textsc{Matlab} deblurrers, but these did much worse than the TV deblurring.} the TV-based deblurrer but with 500 split Bregman iterations and tolerance $10^{-10}$; the BM3D deblurrer from Mäkinen \emph{et al.} \cite{MakinenAzzariFoi2020} (i.e., \texttt{BM3DDEB} in the associated software); and a BM3D denoising followed by a TV deblurring (with fidelity term 150, 100 split Bregman iterations, and tolerance $10^{-7}$). Typical deblurrings via these methods are shown in \Cref{fig:deblurrers}.
\begin{figure}[ht]
	\centering
	\includegraphics[width = \textwidth]{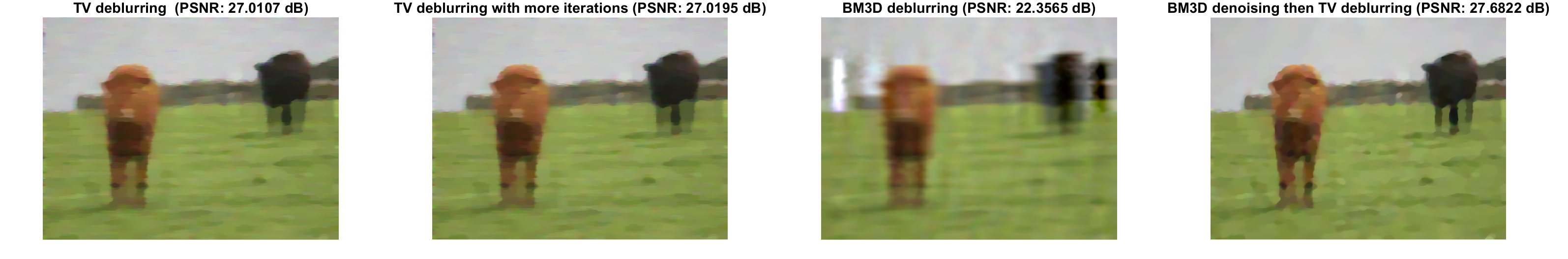}
	\caption{Typical deblurred output for the TV-based deblurrings, BM3D deblurring, and BM3D denoising followed by TV deblurring. \label{fig:deblurrers}}
\end{figure}
We segment only the latter of these, as it is the only one with a perceptible improvement over the TV deblurring.  The mean Dice score $\pm$ standard deviation (over 50 trials) for MBO segmentations of this deblurred image (with $\tau = \varepsilon = 0.00285$, $K = 200$, and initial state $0.45\chi_Y + f$, cf. \Cref{fig:results2blurr}) is $0.8737\pm 0.0143$.
\footnote{Taking $\tau=\varepsilon=0.002$ and initial state $0.47\chi_Y +f$ (as in \Cref{sec:blurparams}) performs slightly worse, with mean Dice score $\pm$ standard deviation (over 50 runs) equal to $0.8679\pm 0.0095$.
}   Segmentations from the first three of these runs are shown in \Cref{fig:BM3D+TV_seg}. Visually, these segmentations appear slightly patchier than those we obtain with our joint method (cf. \Cref{fig:RSdeblur}), but they have higher Dice scores (with the exception of the metastable optimal segmentation observed in \Cref{fig:results4blurr}, which has a slightly higher mean Dice score). However, the sequential deblurring-segmentations are much faster than the joint method;  over 50 runs, their mean ($\pm$ standard deviation)  reconstruction time is $23.45 (\pm 0.16)s$, segmentation time $5.85 (\pm 0.05)s$, and total time $29.30 (\pm 0.17)s$.

\begin{figure}[h]
	\centering 
	\includegraphics[width = \textwidth]{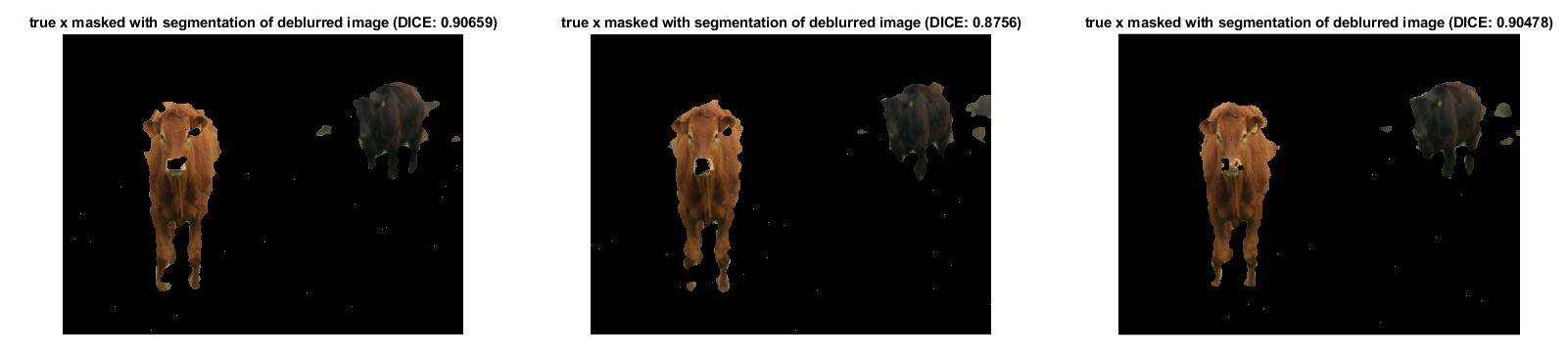}
	\caption{The true image masked with three typical segmentations of the BM3D denoised and TV deblurred image.} \label{fig:BM3D+TV_seg} 
\end{figure}
\section{Conclusions and directions for future work}

In this paper, we have developed a joint reconstruction-segmentation scheme which incorporates the highly effective graph-based segmentation techniques that have been developed over the past decade. There are numerous challenges which arise in the efficient implementation of this scheme, but we have shown how these obstacles can be navigated. Furthermore, we have shown how the Kurdyka--\L ojasiewicz-based theory of \cite{attouch2010proximal,BoltePALM} can be applied to show the convergence of our scheme.

Finally, we have tested our scheme on highly-noised and blurry counterparts of the ``two cows'' image familiar from the literature. In the denoising case, we observed that our scheme gives very accurate segmentations despite the very high noise level, and gives reasonably accurate reconstructions, with a run time of about 2.5 minutes. Moreover, our joint scheme substantially outperforms sequential denoising-segmention methods, in both segmentation and reconstruction accuracy, even when much more sophisticated denoisers are employed, albeit at the cost of a much longer run time. 

In the deblurring case, again our scheme gives highly accurate segmentations (with a run time of about 3 minutes), but in the reconstructions it introduces an artificial level of contrast between the ``cows'' and the ``background''. This aids the segmentation accuracy at the cost of the reconstruction accuracy deteriorating over the course of the iterations. Increasing the reconstruction weighting prevents this effect, at the expense of a lower segmentation accuracy. Increasing the segmentation weighting has the curious effect of producing a very accurate but metastable (w.r.t. a change in the number of iterations) segmentation after 9 iterations. Compared to sequential deblurring-segmentation, our scheme produces worse reconstructions and on the whole slightly worse segmentations (with the exception of the metastable segmentation, which is slightly better) and runs considerably slower. However it should be noted that the deblurring method which was used is more sophisticated than the one used within our scheme, and that our joint scheme does give substantially more accurate segmentations than sequential deblurring-segmentation using only a TV-based deblurrer (i.e., our initialisation process). 

There are three major directions for future work. First, the scheme in its current form has many parameters which must be tuned by hand. Future work will seek to develop techniques for tuning these parameters in a more principled way, so that this scheme can be applied to a large image set without the need for constant manual re-tuning. 

Second, we have in this work applied our scheme only to artificially noised/blurred images, and in the comparison to sequential methods we did not exhaust the state-of-the-art. Future work will seek to test our scheme on real observations, with potentially unknown ground truths and/or forward maps, and compare our scheme to other state-of-the-art methods (including other joint reconstruction-segmentation methods such as those in Corona \emph{et al.} \cite{Corona2019}). 

Finally, there are a number of potential ways to make our scheme more accurate. One is to use a different regulariser. Candidates of particular interest are implicit regularisation with a ``Plug-and-Play'' denoiser (as in Venkatakrishnan \emph{et al.} \cite{PNP}) and learned regularisation (see e.g. Arridge \emph{et al.} \cite[\S 4]{Arridge2019}). Another potential improvement is the use of more sophisticated feature maps (e.g., the CNN-VAE maps used in Miller \emph{et al.} \cite{Miller}), though in this choice one is constrained by the need to compute both the map and its adjoint very efficiently. Moreover, the theory in this paper assumes a linear feature map. A third opportunity for improvement lies in solving \cref{xupdate2} without resorting to linearisation. 

\appendix

\section{Computing $g_n$}\label{sec:g}
In this appendix we explain how we compute $g_n$, which is defined in \Cref{xupdatesec}. We recall from \Cref{sec:setup} that $z:=\mathcal{F}(x)$, and likewise we define $z_n:=\mathcal{F}(x_n)$. Then 
\[
 g_n = \mathcal{F}^*(\nabla_z\mathscr{F}_1(z_n) + \nabla_z\mathscr{F}_2(z_n)).
\]
We now compute each term in turn.
To compute $\nabla_z\mathscr{F}_1(z)$, note that 
\begin{align*}
	\mathscr{F}_1(z+\delta z) &=\beta \langle{(G_n)_{YY},\Omega_{YY}(z+\delta z)} \rangle_F
	\\&
	= \beta\left\langle{(G_n)_{YY},\Omega_{YY}(z) + \left[\langle\nabla_z\Omega_{ij}(z),\delta z\rangle_F\right]_{i,j\in Y}} \right\rangle_F +o(\delta z)\\
	&= \mathscr{F}_1(z) + \beta \sum_{i,j\in Y} (G_n)_{ij}\sum_{l \in Y}\sum_{r=1}^q (\nabla_z \Omega_{ij}(z))_{lr} \delta z_{lr} + o(\delta z)
\end{align*}
and thus for all $l\in Y$ and $r \in \{1,...,q\}$
\[
(\nabla_z\mathscr{F}_1(z))_{l r} =\beta \sum_{i,j\in Y} (G_n)_{ij} (\nabla_z\Omega_{ij}(z))_{l r}.
\]
Now, since $(\Omega_{YY})_{ij}(z) = e^{-\|z_i - z_j\|_2^2/q\sigma^2}$ for $i\neq j$ and $(\Omega_{YY})_{ii}(z)=0$, we have  
\begin{equation*}
	\nabla_{z_{l r}} (\Omega_{YY})_{ij}(z) =\frac{2}{q\sigma^2} \left.\begin{cases} 0, &\hspace{-1em}l \notin \{i,j\} \\
		(\Omega_{YY})_{il}(z)(z_{ir} -z_{l r}), &j = l\\
		(\Omega_{YY})_{l j}(z)(z_{l r}-z_{jr}), &i=l
	\end{cases}\hspace{-0.05cm}\right\} 
	= \frac{2}{q\sigma^2}(\Omega_{YY})_{ij}(z)(z_{ir}-z_{jr})(\delta_{jl}-\delta_{il}),
\end{equation*}
where $\delta$ denotes the Kronecker delta. Therefore,
\[
(\nabla_z\mathscr{F}_1(z))_{l r} = \frac{2\beta}{q\sigma^2}\sum_{i,j\in Y} (G_n)_{ij} \Omega_{ij}(z)(z_{ir}-z_{jr})(\delta_{jl}-\delta_{il}).
\]
Hence, letting $\mathcal{A}(z):=(G_n)_{YY}\odot \Omega_{YY}(z)$ (where $\odot$ denotes the Hadamard product), 
\begin{equation*}
	(\nabla_z\mathscr{F}_1(z))_{l r}= \frac{2\beta}{q\sigma^2}\sum_{i,j\in Y} \mathcal{A}_{ij}(z)(z_{ir}-z_{jr})(\delta_{jl}-\delta_{il})
	=  \frac{4\beta}{q\sigma^2} \left(\sum_{j \in Y} \mathcal{A}_{lj}(z)z_{jr} - z_{lr} \sum_{j\in Y} \mathcal{A}_{lj}(z) \right)
\end{equation*}
since $\mathcal{A}(z)$ is symmetric, and therefore
\begin{equation*}
\nabla_z\mathscr{F}_1(z) = \frac{4\beta}{q\sigma^2} \left( \mathcal{A}(z)z - (\mathcal{A}(z)\mathbf{1}_N)\odot z \right) .
\end{equation*}

To compute $\nabla_z\mathscr{F}_2(z)$, note that by a similar argument as the above,
\[
(\nabla_z\mathscr{F}_2(z))_{l r} = -\frac{4\beta}{q\sigma^2} \sum_{i\in Y,j\in Z}  (G_n)_{ij}(\Omega_{YZ})_{ij}(z,z_d)(z_{ir}-(z_d)_{jr})\delta_{il}, 
\]
 for all $l \in Y$ and $r \in \{1,...,q\}$. Hence, letting $\mathcal{B}(z):=(G_n)_{YZ}\odot \Omega_{YZ}(z,z_d)$, we get 
\begin{equation*}
	(\nabla_z\mathscr{F}_2(z))_{l r}= \frac{4\beta}{q\sigma^2}\sum_{i\in Y,j\in Z} \mathcal{B}_{ij}(z)((z_d)_{jr}-z_{ir})\delta_{il} 
	= \frac{4\beta}{q\sigma^2}\left( \sum_{j\in Z}\mathcal{B}_{lj}(z)(z_d)_{jr} - z_{lr}\sum_{j\in Z}\mathcal{B}_{lj}(z) \right)
\end{equation*}
and therefore we arrive at a similar formula as for $\nabla_z\mathscr{F}_1(z)$:
\begin{equation*}
\nabla_z\mathscr{F}_2(z) = \frac{4\beta}{q\sigma^2} \left( \mathcal{B}(z)z_d - (\mathcal{B}(z)\mathbf{1}_{N_d})\odot z \right) .
\end{equation*}

Tying this all together, we get 
\be
\label{gradF2}
\begin{split}
	g_n &= \frac{4\beta}{q\sigma^2}\mathcal{F}^*\big( \mathcal{A}(z_n)z_n+\mathcal{B}(z_n)z_d -  (\mathcal{A}(z_n)\mathbf{1}_N +\mathcal{B}(z_n)\mathbf{1}_{N_d})\odot z_n \big)\\
	&=\frac{4\beta}{q\sigma^2} \mathcal{F}^*\left( \mathcal{C}_n\begin{pmatrix}z_n\\z_d\end{pmatrix} -  (\mathcal{C}_n\mathbf{1}_{N+N_d})\odot z_n\right),
\end{split}
\ee
where $\mathcal{C}_n:=(G_n)_{YV}\odot \Omega_{YV}(z_n,z_d)$. 
To compute \cref{gradF2}, we need to compute matrix-vector products of the form $\mathcal{C}_nv$. 
Recalling \cref{Gn}, it follows that 
\[
(G_n)_{YV} = -u_n|_Yu_n^T + v_n|_Y\mathbf{1}_V^T  + \mathbf{1}_Yv_n^T. 
\]
Next, we observe a neat linear algebra result\footnote{To see this, observe that  in suffix notation, for $i\in Y$ and $j\in V$, the LHS is $(-(u_n)_i(u_n)_j + (v_n)_i + (v_n)_j)A_{ij}v_j$ and the RHS is $-(u_n)_i A_{ij}((u_n)_jv_j) + (v_n)_i A_{ij} v_j + A_{ij}((v_n)_j v_j)$.}
\[
\left((-u_n|_Yu_n^T + v_n|_Y\mathbf{1}_V^T  + \mathbf{1}_Yv_n^T )\odot A\right)v = -u_n|_Y \odot (A(u_n \odot v)) + v_n|_Y \odot (Av) + A(v_n\odot v),
\]where in this case $A =  \Omega_{YV}(z_n,z_d)$. Hence it suffices to be able to compute terms of the form $\Omega_{YV}(z_n,z_d)v $. Via the Nystr\"om extension \cref{Nys} we have 
\be\label{Omv}
\Omega_{YV}(\mathcal{F}(x_n),z_d)v \approx \left( \Omega_{VX}(\mathcal{F}(x_n),z_d)\Omega_{XX}^{-1}(\mathcal{F}(x_n),z_d)\Omega_{XV}(\mathcal{F}(x_n),z_d) v \right)|_{Y},
\ee
where $X\subseteq V$ is some interpolation set, so we can compute such products quickly.

These considerations lead us to \Cref{Cz} to compute $\mathcal{C}_n v$ for $v\in \mathcal{V}$.

\begin{algorithm}[h]
	\caption{Definition of the \texttt{CProd} function to be used in \Cref{RSalg2}.\label{Cz}} 
	\begin{algorithmic}[1]
\Function{CProd}{$z_n,v,u,v_n,\sigma,V,Y,Z,K$} \Comment{Approximates $\mathcal{C}_n v$ as above}
\State $B: w \mapsto (\texttt{OmegaProd}(w,z_n,z_d,q,\sigma,V,Y,Z,K))|_Y$ \Comment{See below}
\State \textbf{return} $-u|_Y \odot B(u \odot v) + v_n|_Y \odot B(v) +  B(v_n \odot v)$
\EndFunction
\Function{OmegaProd}{$v,z,z_d,q,\sigma,V,Y,Z,K$} \Comment{
	\parbox[t]{.4\linewidth}{
Approximates $\Omega(z,z_d)v$ via the Nystr\"om extension as in \cref{Omv}
}
}
\State $\omega: ij \mapsto \Omega_{ij}(z,z_d,q,\sigma)$ \Comment{Defined as in \cref{Omega}}
\State $ {X} = \texttt{random\_subset}(Y,K/2) \cup \texttt{random\_subset}(Z,K/2) $ 
\State $(\omega_{XX},\omega_{VX}) = (\omega(X,X),\omega(V,X))$
\State $\omega_{XX}v' = \omega_{VX}^T v$ \Comment{Solving the linear system for $v'$}
\State \textbf{return} $\omega_{VX}v'$
\EndFunction
	\end{algorithmic}
\end{algorithm}

\section{Primal-dual optimisation methods for \cref{xupdate3}}\label{primaldual}

To solve \cref{xupdate3}, we shall be employing an algorithm of Chambolle and Pock \cite{ChPock}.
Following \cite{ChPock}, we rewrite \cref{xupdate3} as:
\be \label{CPprimal}
\min_{x \in \mathbb{R}^{N \times \ell}} R(\mathcal{K}x) + G(x), 
\ee
where $G(x):= \alpha \|\mathcal{T}(x)-y\|_F^2 + \eta_n\|x-\tilde x_n\|_F^2$.
Then the primal problem \cref{CPprimal} can be reformulated as the primal-dual saddle point problem (we recall \Cref{foot:R*} for the definition of $R^*$)
\be \label{CPsaddle}
\min_{x \in \mathbb{R}^{N \times \ell}} \max_{p \in \mathcal{K}( \mathbb{R}^{N \times \ell})} \: \langle \mathcal{K}x, p \rangle + G(x) - R^*(p).
\ee
Reformulating \cref{CPprimal} as \cref{CPsaddle} suggests approaching the minimiser of \cref{CPprimal} by alternately updating a sequence of $x_n$ and $p_n$. 
The algorithm we shall be using, i.e. \cite[Algorithm 2]{ChPock}, is a sophistication of this basic idea, and is summarised as \Cref{pdalg}. For details and a convergence analysis, see \cite{ChPock}. 
\begin{algorithm}[h]
	\caption{Algorithm for solving \cref{CPprimal}, using  \cite[Algorithm 2]{ChPock}.\label{pdalg}}
	\begin{algorithmic}[1]
		\Function{PrimalDual}{$x_0, \gamma, \mathcal{K}, \mathcal{K^*}, \texttt{proxRS},\texttt{proxG}$}
		\Comment{$\gamma$ must satisfy \cite[(35)]{ChPock}} 
		\State $p_0 = \mathcal{K}(x_0)$
		\State $\bar x_0 = x_0$
		\State $ (\delta t_{0,1},\delta t_{0,2}) = (1/\|\mathcal{K}\|,0.99/\|\mathcal{K}\|)$ \Comment{We must have $\delta t_{0,1}\delta t_{0,2} \|\mathcal{K}\|^2\leq 1$}
		\State $n=0$
		\While{stopping condition not met}
		\State $ p_{n+1} = \texttt{proxRS}(p_n + \delta t_{n,1} \mathcal{K} (\bar x_n),\delta t_{n,1})$ \Comment{$\texttt{proxRS}(p,\delta t):= \operatorname{prox}_{\delta t R^*}(p)$}
		\State $ x_{n+1} = \texttt{proxG}(x_n - \delta t_{n,2} \mathcal{K}^* (p_{n+1}),\delta t_{n,2})$ \Comment{$\texttt{proxG}(x,\delta t):= \operatorname{prox}_{\delta t G}(x)$}
		\State $ \theta_n = (1 + 2\gamma \delta t_{n,2})^{-\frac12} $
		\State $ (\delta t_{n+1,1}, \delta t_{n+1,2}) = (\delta t_{n,1}\theta_n ^{-1},\delta t_{n,2}\theta_n) $
		\State $ \bar x_{n+1} = x_{n+1} + \theta_n(x_{n+1}-x_n)$
		\State $ n = n + 1$
		\EndWhile
		\State \textbf{return} $x_n$
		\EndFunction
	\end{algorithmic}
\end{algorithm}
We shall assume that for any $\delta t > 0$ we can efficiently compute the proximal operator of $\delta t R^*$, $\operatorname{prox}_{\delta t R^*}$.\footnote{For a proper, l.s.c., and convex function $f$, the \emph{proximal (a.k.a. resolvent) operator of $f$} is defined by $\operatorname{prox}_f(x'):=\argmin_x f(x) + \frac12\|x-x'\|^2$. As mentioned in \cite{ChPock}, Moreau's identity allows $\operatorname{prox}_{\delta t R^*}$ to be computed via $\operatorname{prox}_{\delta t^{-1}R}$. Since computing $\operatorname{prox}_{\delta t^{-1}R}$ is equivalent to performing a denoising regularised by $R$, this is a place where ``Plug-and-Play'' denoising methods (see Venkatakrishnan \emph{et al.} \cite{PNP}) could be employed instead of the explicit regularisation methods we will use in this paper. } So to employ \Cref{pdalg}, we need two ingredients: a method for computing $\operatorname{prox}_{\delta t G}$, and a $\gamma$ obeying \cite[(35)]{ChPock}, i.e. such that for all $x,x'\in \mathbb{R}^{N \times \ell}$, 
\begin{equation*} G(x') \geq G(x) + \langle \nabla_x G(x), x'-x \rangle_F + \frac12 \gamma \|x - x'\|_F^2.\end{equation*}

First, assume that $\mathcal{T}$ is linear. Then it follows that $
\nabla_x G(x) = 2\alpha \mathcal{T}^*(\mathcal{T}(x)-y) + 2\eta_n(x-\tilde x_n)$
and thus we require that
\begin{align*}
  \frac{ 1}{2} \gamma\|x - x'\|_F^2 & \leq 
  \hspace{0cm}\alpha \langle \mathcal{T}(x)+\mathcal{T}(x')-2y, \mathcal{T}(x') - \mathcal{T}(x)\rangle_F + \eta_n \langle x + x' - 2\tilde x_n, x'-x \rangle_F \\
  &\hspace{0cm}\quad-  \langle 2\alpha \mathcal{T}^*(\mathcal{T}(x)-y) + 2\eta_n(x-\tilde x_n), x'-x \rangle_F  \\
 &\hspace{0cm}= \alpha \left( \langle \mathcal{T}(x)+\mathcal{T}(x')-2y, \mathcal{T}(x' - x)\rangle_F - 2 \langle\mathcal{T}(x)-y,\mathcal{T}(x'-x) \rangle_F \right) 
 + \eta_n\|x - x'\|_F^2\\
 &\hspace{0cm}= \alpha \|\mathcal{T}(x' - x)\|^2_F + \eta_n\|x - x'\|_F^2, 
\end{align*}
and hence it suffices to take $\gamma = 2\eta_n$. Finally, recall that $\operatorname{prox}_{\delta t G}$ is defined by 
\[
\operatorname{prox}_{\delta t G}(x) = \argmin_{x'\in \mathbb{R}^{N \times \ell}} G(x') + \frac{\|x'-x\|_F^2}{2\delta t}
\]
which has unique minimiser $x'$ solving $\nabla_x G(x') + \frac{1}{\delta t}(x'-x) = 0$. Solving for $x'$, we get: 
\be\label{proxG}
 \operatorname{prox}_{\delta t G}(x) =  \left((\delta t^{-1}+2\eta_n)I + 2\alpha \mathcal {T}^* \mathcal{T}\right)^{-1}\left(2\alpha\mathcal{T}^*(y)+2\eta_n\tilde x_n+\frac{x}{\delta t}\right).
\ee
If $\mathcal{T}$ is non-linear, things are more difficult. There may not exist a valid $\gamma$, and if so one must use a method such as \cite[Algorithm 1]{ChPock} to solve \cref{xupdate3}, which has slower convergence. We must still compute  $\operatorname{prox}_{\delta t G}(x)$. Since $\mathcal{T}$ is assumed to be differentiable, for all $x$ there exists a linear map $D\mathcal{T}_x$ such that $\mathcal{T}(x + \delta x) = \mathcal{T}(x) + D\mathcal{T}_x(\delta x) + o(\delta x)$. Then it follows that 
\[
\nabla_x G(x) = 2\alpha (D\mathcal{T}_x)^*(\mathcal{T}(x)-y) + 2\eta_n(x-\tilde x_n)
\] 
and hence $x' := \operatorname{prox}_{\delta t G}(x)$ solves 
\begin{equation}\label{proxGnonlin}
	 (1+2\delta t \eta_n )x' + 2\alpha \delta t (D\mathcal{T}_{x'})^*(\mathcal{T}(x')-y) = x + 2\delta t \eta_n\tilde x_n.
\end{equation} 
Finally, \cref{proxGnonlin} can be solved by numerical root-finding methods (see e.g. \cite[\S9]{NumRec}). 

\section{Computing the SDIE scheme}\label{sec:computeSDIE}
In this appendix we describe how we compute the SDIE scheme that is described in \Cref{subsec:SDIE}. By \Cref{fSDsolnthm}, an SDIE update has two steps: a fidelity-forced diffusion and a piecewise linear thresholding. The thresholding is trivial to compute, but the diffusion is non-trivial. Our method for computing fidelity-forced diffusion was described in detail in \cite[\S 5.2.6]{Buddthesis}, so here we only reproduce the key details. 

By \Cref{fdiffusethm}, given $u_n$ and the parameter $\tau>0$, we seek to compute
\[
\mathcal{S}_\tau u_n =	e^{-\tau(\Delta + M')}u_n + b,
\]
where $b:=F_\tau(\Delta + M')M'f'$. To compute $e^{-\tau(\Delta + M')}u_n $, we use the Strang formula \cite{Strang}:
\begin{equation*} 
    e^{-\tau(\Delta + M')} = \left( e^{-\frac{\tau}{2m}M'} e^{-\frac{\tau}{m} \Delta} e^{-\frac{\tau}{2m}M'}\right)^m + \bigO\left({m^{-2}}\right) 
	= \left( e^{-\frac12 \delta t M'} e^{-\delta t \Delta}e^{-\frac12 \delta t M'}\right)^m + \bigO\left(\delta t^2\right),
\end{equation*}
where $m \in \mathbb{N}$ and $\delta t := \tau/m$. Computing matrix-vector products with $e^{-\frac12 \delta t M'}$ is straightforward, since $M'$ is a diagonal matrix. To compute matrix-vector products with $e^{-\delta t \Delta}$, we will compute an approximate eigendecomposition of $\Delta$ using the Nystr\"om-QR method 
 (recommended by Alfke \emph{et al.} \cite{APSV2018}) 
described in \Cref{nysQR}. 
For details on this method, see \cite[\S5.2.3]{Buddthesis}.
\begin{nb}\label{nb:Nys}
 \Cref{nysQR} really computes an approximate decomposition $ U_s \Sigma U_s^T$ of $\tilde \omega := D^{-1/2} \omega D^{-1/2}$, and then makes a further approximation
	$
	\Delta_s = I - \tilde\omega \approx U_s(I_{K}-\Sigma)U_s^T = U_s\Lambda U_s^T
	$,
	where $I_K$ is the $K\times K$ identity matrix, and so on for the random walk Laplacian.
\end{nb}
\begin{algorithm}[h]
	\caption{Nystr\"om-QR method for computing an approximate SVD of $\Delta$ or $\Delta_s$.
		\label{nysQR}}
	\begin{algorithmic}[1]
		\Function{Nystr\"omQR}{$ij\mapsto\omega_{ij}, V, Z, K_1,K_2$}\Comment{Computes $U_1,\Lambda$, and $U_2$, where $\Delta\approx U_1\Lambda U_2^T$ is an  \\ \hfill approximate SVD of rank $K:=K_1+K_2$  }
		\State $ {X} = \texttt{random\_subset}(V\setminus Z,K_1) \cup  \texttt{random\_subset}(Z,K_2)$ 
\State $(\omega_{XX},\omega_{VX}) = (\omega(X,X),\omega(V,X))$
		\State $\hat d = \omega_{V{X}} \left(\omega_{{XX}}^{-1} \left(\omega_{V{X}}^T\mathbf{1}\right)\right) $ \Comment{Uses \cref{Nys} to approximate $d = \omega\mathbf{1}$}
		\State $ \tilde\omega_{VX} =\hat d^{-1/2} \odot \omega_{VX}$ \Comment{Applying $\odot$ columnwise, i.e. $(\tilde\omega_{V{X}})_{ij} = \hat d_i^{-1/2}(\omega_{V{X}})_{ij} $ }
		\State $ [Q,R] = \texttt{thin\_qr}(\tilde \omega_{V{X}}) $\Comment{Computes thin QR decomposition $\tilde \omega_{V{X}}=QR$}
		\State $ S = R\omega_{{XX}}^{-1}R^T$ \Comment{N.B. $S \in \mathbb{R}^{K \times K}$}
		\State $ S = (S + S^T)/2 $ \Comment{Corrects symmetry-breaking computational errors}
		\State $ [\Phi,\Sigma] = \texttt{eig}(S)$\Comment{Computes eigendecomposition $S = \Phi\Sigma\Phi^T$}
		\State $ \Lambda = I_K -\Sigma$ 
		\State $ U_s = Q\Phi$ \Comment{
				$\Delta_s\approx U_s \Lambda U_s^T$, so to return the decomposition of $\Delta_s$ terminate here}
		\State $ U_1 = \hat d^{-1/2} \odot U_s $ \Comment{I.e. $(U_1)_{ij} =  \hat d_i^{-1/2} (U_s)_{ij} $}
		\State $  U_2 = \hat d^{1/2} \odot U_s $ \Comment{I.e. $(U_2)_{ij} =  \hat d_i^{1/2} (U_s)_{ij} $}
		\State \textbf{return} $U_1,\Lambda,U_2$
		\EndFunction
	\end{algorithmic}
\end{algorithm}

In \cite{Budd3,Buddthesis}, this Nystr\"om-QR approximate decomposition was used to approximate the matrix exponential via $e^{-\delta t \Delta} \approx I + U_1(e^{-\delta t \Lambda} - I_K)U_2^T$. 
But by \Cref{nb:Nys}, $\Delta \approx I - U_1\Sigma U_2^T$ is a slightly more accurate approximation, and so we have the improved approximation:
\[
 e^{-\delta t \Delta} \approx e^{-\delta t} \left(I + U_1(e^{\delta t \Sigma} - I_K)U_2^T\right).
\]
Therefore, we compute $v_m \approx e^{-\tau(\Delta + M')}u_n$ by defining $v_0 := u_n$, and $v_r$ for $r \in\{1,...,m\}$ by 
\be \label{eq:Strang}
\begin{split}
	v_{r+1} &= e^{-\delta t} e^{-\delta t M'}v_r + e^{-\delta t}e^{-\frac{1}{2}\delta t M'} U_1(e^{\delta t \Sigma}-I_K)U_2^T e^{-\frac{1}{2}\delta t M'}v_r\\
	&= a_1(\delta t) \odot v_r + a_3(\delta t) \odot \left( U_1\left( a_2(\delta t) \odot \left(U_2^T\left(a_3(\delta t) \odot v_r\right) \right)\right) \right)
\end{split}
\ee
where $a_1(\delta t):= \operatorname{exp}(-\delta t (\mu'+\mathbf{1}))$, $a_2(\delta t) := \operatorname{exp}(\delta t\operatorname{diag}(\Sigma))-\mathbf{1}_K$, and $a_3(\delta t)$ is the elementwise square root of $a_1(\delta t)$ 
(where $\exp$ is applied elementwise, and $\mathbf{1}_K$ is the vector of $K$ ones).

Finally we note that, by \Cref{fdiffusethm}, $b$ is the fidelity-forced diffusion with initial condition $u_0=\mathbf{0}$ at time $\tau$. We compute $b$ via the semi-implicit Euler scheme used for fidelity-forced diffusion in \cite{MKB}. To compute this, again we use the Nystr\"om-QR decomposition of $\Delta$.

\begin{nb}
	We do not use the scheme from \emph{\cite{MKB}} for all of the fidelity-forced diffusions because the Strang formula is more accurate for the $e^{-\tau(\Delta + M')}u_n$ term, see \emph{\cite[\S5.2.5--5.2.7]{Buddthesis}} for details. 
\end{nb}

\section{Examples of sub-analytic regularisers}\label{app:subany}

The following theorem shows that the examples of regularisers $\mathcal{R}$ that are given in \Cref{note:assumptionsRT}, do indeed satisfy the conditions required by \Cref{ass:R}.

\begin{thm}\label{thm:app1}
	The following functions are semi-analytic, bounded below, and continuous on their domain (which is $\mathbb{R}^n$):
	\begin{enumerate}[i.]
		\item $f:x \mapsto\|Ax\|_1$ for $A\in \mathbb{R}^{m\times n}$. 
		\item A feedforward neural network $\mathcal{NN}:=f_L \circ \cdots \circ f_1$ where $f_j(x') := \rho(A^{(j)}x' + b^{(j)} )$, $A^{(j)}\in \mathbb{R}^{n_j\times n_{j-1}}$, $b^{(j)} \in \mathbb{R}^{n_j}$, $n_0 = n$, $n_L =1$, and $(\rho(x'))_i:=\max\{0,x'_i\}$ (the ReLU function).  
	\end{enumerate}
\end{thm}
\begin{proof}
\begin{enumerate}[i.]
\item $\operatorname{Gr} f = \{(x,\|Ax\|_1)\mid x \in \mathbb{R}^n\}$ can be written as 
\[
\bigcup_{e\in \{-1,1\}^m} \bigcap_{j=1}^m \left\{(z_1,z_2) \in \mathbb{R}^{n+1} \: \middle | \: z_2 - \sum_{i=1}^m e_i(Az_1)_i = 0 \text{ and } e_j(Az_1)_j \geq 0  \right\}
\]
and thus $\operatorname{Gr} f$ is a semi-analytic set. The other properties are trivial.
\item $\mathcal{NN}$ is a composition of piecewise linear continuous functions, and is hence piecewise linear and continuous. It follows that it is semi-analytic.
It is bounded below due to $\rho$ applying the ReLU function. 
\end{enumerate}
\end{proof}

\bibliographystyle{siamplain}
 \bibliography{dissertation}
\end{document}